%% file: main.tex
\documentclass[nohyperref]{article}

\usepackage{graphicx}
\usepackage{booktabs}

\usepackage{hyperref}

\usepackage[accepted]{icml2022}

\usepackage{subcaption}
\usepackage{graphicx}
\usepackage{amsmath,amsfonts,amsthm} 
\usepackage{mathtools}
\usepackage{enumitem}
\usepackage{multirow}
\usepackage[capitalize,noabbrev]{cleveref}

\theoremstyle{plain}
\newtheorem{theorem}{Theorem}[section]
\newtheorem{proposition}[theorem]{Proposition}
\newtheorem{lemma}[theorem]{Lemma}
\newtheorem{corollary}[theorem]{Corollary}
\theoremstyle{definition}
\newtheorem{definition}[theorem]{Definition}
\newtheorem{assumption}[theorem]{Assumption}

\usepackage[textsize=tiny]{todonotes}

\input{definitions}

\icmltitlerunning{Linear Bandit Algorithms with Sublinear Time Complexity}

\begin{document}

\twocolumn[
\icmltitle{Linear Bandit Algorithms with Sublinear Time Complexity}

\icmlsetsymbol{equal}{*}

\begin{icmlauthorlist}
\icmlauthor{Shuo Yang}{utcs}
\icmlauthor{Tongzheng Ren}{utcs}
\icmlauthor{Sanjay Shakkottai}{utee}
\icmlauthor{Eric Price}{utcs}
\icmlauthor{Inderjit S. Dhillon}{utcs}
\icmlauthor{Sujay Sanghavi}{utee}
\end{icmlauthorlist}

\icmlaffiliation{utcs}{Department of CS, The University of Texas at Austin, TX, USA.}
\icmlaffiliation{utee}{Department of ECE, The University of Texas at Austin, TX, USA.}

\icmlcorrespondingauthor{Shuo Yang}{yangshuo\_ut@utexas.edu}

\icmlkeywords{Bandits, Sublinear Time}

\vskip 0.3in
]

\printAffiliationsAndNotice{}

\input{abstract}
\input{introduction}
\input{related_work}
\input{max_inner_product_search}

\input{problem_setup}
\input{sublinear_elimination}
\input{accelerate_linTS}
\input{experiments}

\bibliography{example_paper}
\bibliographystyle{icml2022}

\newpage
\appendix
\onecolumn

\input{apdx_proof}

\input{apdx_experiments}

\end{document}

%% file: abstract.tex
\begin{abstract}

We propose two linear bandits algorithms with per-step complexity sublinear in the number of arms $K$. The algorithms are designed for applications where the arm set is extremely large and slowly changing. Our key realization is that choosing an arm reduces to a maximum inner product search (MIPS) problem, which can be solved approximately without breaking regret guarantees. Existing approximate MIPS solvers run in sublinear time. We extend those solvers and present theoretical guarantees for online learning problems, where adaptivity (i.e., a later step depends on the feedback in previous steps) becomes a unique challenge. We then explicitly characterize the tradeoff between the per-step complexity and regret. For sufficiently large $K$, our algorithms have sublinear per-step complexity and $\widetilde O(\sqrt{T})$ regret. Empirically, we evaluate our proposed algorithms in a synthetic environment and a real-world online movie recommendation problem. Our proposed algorithms can deliver a more than 72 times speedup compared to the linear time baselines while retaining similar regret.

\end{abstract}

%% file: introduction.tex
\section{Introduction}

Linear bandits problem is one of the most fundamental online learning problems, with wide applications in recommender systems, online advertisements, etc. \citep{deshpande2012linear}. Such applications usually have an extremely large set of items (e.g., millions of products to be recommended), which also changes over time. Specifically, we focus on two types of changes: (1) some new arms are added from time to time (e.g., new movies added to the database); and more generally (2) some new arms are added, and some old arms deleted (e.g., some new advertisements to be shown and some old ones expired). Such an extremely large arm set typically changes slowly, in the sense that a relatively small number of arms are added or deleted at every time step.

A linear scan is slow for an extremely large arm set. It is thus demanding to design linear bandit algorithms that have per-step time complexity sublinear in the number of arms $K$, for an extremely large and slowly changing arm set. 

Common algorithms for linear bandits have per-step time complexity linear in $K$. For instance, Thompson Sampling (TS) draws a random parameter estimate and selects the best arm accordingly \citep{abeille2017linear}. It needs to scan the entire set of arms to choose the most promising arm, which leads to time complexity linear in $K$.

In this paper, we propose two algorithms with per-step time complexity sublinear in $K$, based on the observation below:

\textit{\textbf{Key observation: }The arm selection step in many linear bandits algorithms reduces to an (exact) maximum inner product search (MIPS) problem. The right way to approximately solve the MIPS problem, coupled with careful analysis, allows us to achieve sublinear per-step complexity and desired regret guarantees.}

Formally, given a set $P\in \RR^{d}, \abs{P} = K$, and a query $q\in \RR^{d}$, the MIPS problem aims to find the point $p\in P$ that maximizes $p^{\top}q$. The TS algorithm is an immediate example of selecting arms by solving a MIPS problem. For arms $a$ with embedding $x_{a}$, TS algorithm chooses the arm that maximizes $x_{a}^{\top}\widetilde \theta$, for the random $\widetilde \theta$ drawn by TS.

More importantly, the exact solution of the MIPS problem is not necessary for obtaining an $O(\sqrt{T})$ regret bound. Take TS algorithm again as an example, the estimate $\widetilde \theta$ has an estimation error (i.e., $\widetilde \theta \neq \theta^*$, where $\theta^*$ is the true environment parameter that determines reward expectation). By properly controlling the approximate MIPS accuracy, the error of approximately solving MIPS can be smaller than the estimation error of $\widetilde \theta$. The regret will therefore stay in the same order as solving the MIPS exactly.

Many approaches were previously established to approximately solve MIPS with time complexity sublinear in $K$. While it seems promising to adopt those approximate MIPS solvers, there are still two challenges remaining:

\textit{\textbf{Challenge 1.} How to design (and analyze) approximate MIPS solvers for a sequence of adaptive queries?} Queries are adaptive (i.e., later queries depend on the results of previous ones) for online learning problems. Existing probabilistic guarantees for approximate MIPS solvers do not allow the queries to be adaptive. In this paper, we provide an alternative scheme where a query is first rounded to the nearest point in an $\epsilon$-net before sending to the MIPS solver. While this scheme is less accurate for a single query, it allows for a better success guarantee when applied to an adaptive sequence of $T$ queries.

\textit{\textbf{Challenge 2.} How to characterize the connection between per-step time complexity and regret?} Intuitively, a faster approximate MIPS solver is less accurate and thus leads to larger regret, while an exact MIPS solver enjoys an optimal regret but spends much more time. This tradeoff has not been characterized. For the two algorithms in this paper, we characterize this tradeoff, and furthermore, show that it allows for $O(K^{1-\alpha(T)})$ per-step complexity for some $\alpha(T) > 0$ while retaining $\widetilde O(\sqrt{T})$ regret.

As a summary, our \textbf{main contributions} are

\begin{enumerate}[leftmargin=*, itemsep=0.5ex]
  
  \item We formally define the $(c, r, \epsilon)$-MIPS problem (\Cref{def:mips}), and propose a scheme to approximately solve MIPS for a sequence of adaptive queries (\Cref{alg:mips-adaptive}). In \Cref{thm:near-linear-MIPS}, we show that our proposed algorithm has $K^{1+o(1)}$ preprocessing time complexity, $K^{\rho_q + o(\log^{-0.45}K)}$ query time complexity, with $\rho_q < 1$, and $K^{o(1)}$ time complexity for adding a new arm.
  
  \item Building upon \Cref{alg:mips-adaptive}, we propose a sublinear time elimination-based algorithm (\Cref{alg:sublinear_elimination}) and a sublinear time TS-based algorithm (\Cref{alg:accelerated-linTS}). We characterize the tradeoff between the time complexity and regret (\Cref{thm:sublinear_elimination,thm:acc-linTS}). With a proper choice of parameters and sufficiently large $K$, one can obtain $\widetilde O(\sqrt{T})$ regret and sublinear per-step time complexity.

  \item We evaluate our algorithms in a synthetic environment and a real-world movie recommendation problem. Compared with the linear time complexity baselines, our algorithms can offer a 72 times speedup when there are 100,000 arms while obtaining similar regret.

\end{enumerate} 

%% file: related_work.tex
\section{Related Work}\label{sec:related_work}

\textbf{Linear bandits.} Two popular lines of approaches have been proposed for linear bandits: UCB-based and TS-based algorithms. The UCB-based algorithm chooses the arm with the largest plausible (according to the upper confidence bound) expected reward. The first algorithm was proposed by \citet{auer2002using} under the name SupLinRel, and extended by \citet{chu2011contextual} to be SupLinUCB. The algorithms maintain a confidence interval estimation, and eliminate the arms stage-by-stage. Subsequently, \citet{abbasi2011improved} presented an improved confidence bound construction and proposed the OFUL algorithm. It achieves $O(d \sqrt{T} \log T)$ regret bound, which nearly matches the information-theoretic lower bound $\Omega(d \sqrt{T})$ \citep{dani2008stochastic} up-to logarithmic factors.

TS algorithms maintain a posterior distribution of the environment parameter, and sampling from the posterior to determine the best arm.
There is now a rich literature on both 
Bayesian  \citep{russo2014learning,russo2016information} and frequentist \citep{KaKoMu12,agrawal2013thompson,pmlr-v32-gopalan14,abeille2017linear} regret bounds.
Our work is based on the frequentist analysis for linear Thomson Sampling, introduced in  \citep{abeille2017linear}.
For an arm set $\Acal$ with $K$ arms, all previously mentioned algorithms have a $\Theta(K)$ per-step time complexity.

There are previous algorithms that achieve sublinear in $K$ complexity, but do not fit into our setting. \citep{todd2016minimum,lattimore2020learning} show that the ``optimal design" approach has constant per-step complexity, but does not work for a changing arm set. \citep{liau2018stochastic} solves the multi-arm bandits problem with constant per-step complexity and constant space complexity, but the approach does not extend to the linear bandits problem.

\citet{jun2017scalable} considered accelerating a TS and a modified UCB algorithm to have $\widetilde O(K^{\rho})$ per-step time complexity, with $\rho = 1 - o(1)$. Their proposed algorithms, however, need $\Omega(K^{1 + \rho} T)$ time in preprocessing, as they need to build a MIPS solver for each of the steps in $T$ to deal with adaptive queries. There is much room to improve on the near quadratic dependency on $K$. 

\textbf{Max inner product search (MIPS).} There has been a large volume of work on (approximately) solving MIPS \citep{teflioudi2015lemp,shen2015learning,guo2016quantization,li2017fexipro,yu2017greedy,morozov2018non,abuzaid2019index,ding2019fast,tan2019efficient,zhou2019mobius}. It has also been demonstrated that MIPS can be applied to various problems for acceleration, e.g., quadratic regression \citep{yang2019interaction}, conditional gradient methods \citep{xu2021breaking}, sparsification problems \citep{song2022speeding}, reinforcement learning \citep{shrivastava2021sublinear}, and deep learning \citep{spring2017scalable,chen2019slide,chen2019fast,kitaev2020reformer,chen2020mongoose,song2021does,song2021training}.

For our theoretical analysis, we focus on reducing MIPS to the nearest neighbor search (NNS) problem, where various reductions have been previously proposed \citep{shrivastava2014asymmetric,bachrach2014speeding,neyshabur2015symmetric,keivani2018improved}. We then solve the NNS by Locality Sensitive Hashing (LSH) \citep{andoni2006near,har2012approximate,andoni2018approximate,yan2018norm}, for its rigorous theoretical guarantee on sublinear query time. For our experiments, we use HNSW \citep{malkov2018efficient} for its outstanding empirical performance. 

%% file: max_inner_product_search.tex
\section{MIPS Solver for Adaptive Queries}\label{sec:mips}

We start by formally defining the Maximum Inner Product Search (MIPS) problem. Subsequently, we define adaptive queries and show how it breaks existing MIPS solvers. We then propose our solution to adaptive queries, which can convert existing MIPS solvers to work for adaptive queries.

\subsection{MIPS Problem and Sublinear Time Solver}

\begin{definition}[$(c, r, \epsilon)$-MIPS problem]\label{def:mips}
  Let $P\subseteq \RR^d$ be a finite set of points with $\norm{p}_2 \le 1, \forall p\in P$. Let $q \in \RR^d$ be the query with $\norm{q}_2 \le 1$. The $(c, r, \epsilon)$-approximated max inner product search ($(c, r, \epsilon)$-MIPS) aims to find $p\in P$ such that $\inner{q}{p} \ge cr - \epsilon$ if there exists $p^*\in P$ with $\inner{q}{p^*} \ge r + \epsilon$.
\end{definition}

The definition is valid with $r > 0, c \le 1, \epsilon \ge 0$. Intuitively, for any query $q$ with unit norm, the $(c, r, \epsilon)$-MIPS problem defined above looks for a point $p\in P$ with $\inner{p}{q}\ge r$, allowing for $(1-c)$ multiplicative error and $\epsilon$ additive error. See \Cref{fig:illustration} for illustration.

Approximately solving the MIPS problem with sublinear time has been well studied. The next result is adapted from \citep{andoni2017optimal}, which solves $(c, r, 0)$-MIPS in sublinear time with a success probability of at least 0.9.

\begin{proposition}[Single Query MIPS solver $\Scal(c, r, 0)$]\label{prop:non-adaptive}
  For a point set $P\subseteq \RR^d$ with $K$ points, there exists a data structure $\Scal(c, r, 0)$ that solves $(c, r, 0)$-MIPS problem for an arbitrary query $q$ with at least 0.9 probability. It has the following time complexity: \textbf{Preprocessing:} $K^{1+o(1)}$; \textbf{Add a Point to $P$:} $K^{o(1)}$; \textbf{Query:} $K^{\rho_q + o\rbr{\log^{-0.45}K}}$, where $\rho_q = \frac{4c'^2}{(1+c'^2)^2}$ and $c' = \sqrt{\frac{3 - cr}{3 -r}}$.
\end{proposition}

Notice that for $c < 1$, we have $c' > 1$ and $\rho_q < 1$. 

The online nature of linear bandits calls for a MIPS algorithm that can deal with a sequence of \textit{\textbf{adaptive queries}, where the later queries \textbf{depend on previous query results}}.

Such adaptive queries naturally arise when applying a MIPS solver $\Scal$ to online learning problems - as will be discussed in later sections, one can query $\Scal$ with the current parameter estimate $\widetilde\theta_t$ and $\Scal$ returns an arm $a_t$ that should be played. The query $\widetilde\theta_t$ depends on all previously played arms $a_\tau, \tau < t$, which are the results of previous queries.

As we illustrate in the next subsection, the adaptive queries introduce a fundamental challenge that one can not apply union bound to extend the probabilistic guarantee for one query to a sequence of adaptive queries.

\begin{figure}
    \includegraphics[clip, trim={120 200 350 150},width=\linewidth]{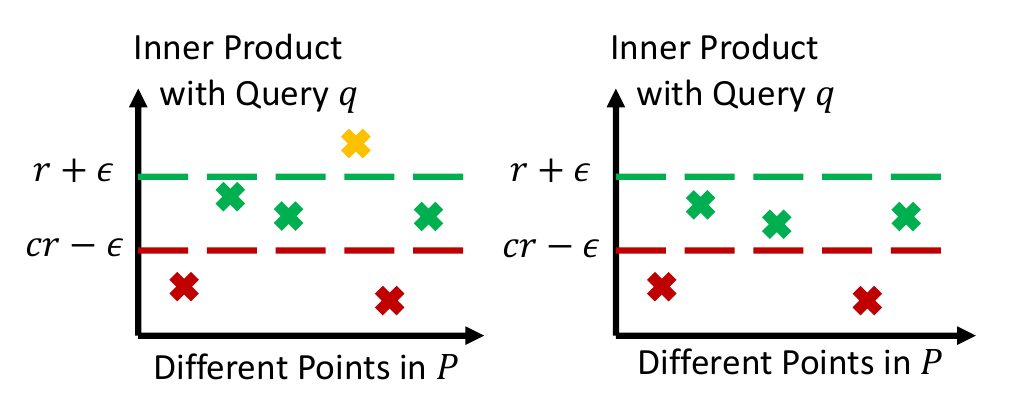}
    \caption{ For query $q$, if there exists $p^* \in P$ that has inner product $\inner{p}{q} \ge r+\epsilon$ (i.e., the yellow point), then the algorithm should return a point $p\in P$ with $\inner{q}{p}\ge cr - \epsilon$ (i.e. green or yellow points in the left figure). Otherwise, no point needs to be returned (i.e., no point needs to be returned for the right-hand side figure).}
    \label{fig:illustration}
    \vspace{-2em}
\end{figure}

\subsection{Hardness of Adaptive Queries}

To see how adaptive queries break union bound, consider the following example. 

\textbf{A Thought Experiment:} A black-box $\Bcal$ has a unit norm vector $p \in \RR^{10}$, drawn uniformly at random when $\Bcal$ is initialized. An agent $\Ccal$ can send unit norm query $q \in \RR^{10}$ to $\Bcal$ and $\Bcal$ returns a scalar $\inner{q}{p}$. Suppose that the agent $\Ccal$ can send 11 queries $q_1, \cdots, q_{11}$ and its goal is to send a query $q^*$ with $\inner{q^*}{p} = 1$. 

For a single query $q$, it is probability $0$ that $q = q^*$, as $p$ is drawn uniformly at random. What is the probability that $\Ccal$ can send such a query $q^*$ within the 11 queries?

\textbf{Adaptive v.s. Non-adaptive:} Consider the two settings - \textit{(1) Non-adaptive queries:} $q_1, \cdots, q_{11}$ can have arbitrary dependency on other queries, but can not depend on any of the results that $\Bcal$ returns; and \textit{(2) Adaptive queries:} a later query $q_i$ can be constructed based on previous queries' result: $\inner{q_j}{p}, j < i$. 

For non-adaptive queries, each $q_i$ has probability 0 to be $q^*$, and thus by union bound, it is probability 0 that $\Ccal$ sends $q^*$ within the 11 queries.

For adaptive queries, $\Ccal$ can first send 10 linearly independent queries $q_1, \cdots, q_{10}$. With the results returned from $\Bcal$, it can solve for $p$ exactly, and send $q_{11} = p$ which gives $\inner{q_{11}}{p} = 1$. Therefore, by allowing the queries to be adaptive, $\Ccal$ can send $q^*$ with probability 1. The drastic difference between probability 0 and probability 1 demonstrates the unique challenge of adaptive queries.

The thought experiment above shows that the probabilistic guarantee for one query cannot be extended to a sequence of adaptive queries via union bound. In the next subsection, we propose a scheme that builds upon $\Scal(c, r, 0)$ and solves MIPS for adaptive queries.

\subsection{MIPS for Adaptive Queries}

The key to solving MIPS with adaptive queries is to discretize the unit $\ell_2$ ball $Q$ (which contains all possible queries) into an $\epsilon$-net $\widehat Q$ and use multiple independent $\Scal(c, r, 0)$ to give correct answers for all queries in $\widehat Q$. 

For a query $q \in Q$, we first round $q$ to its nearest neighbor $\widehat q \in \widehat Q$, which is at most $\epsilon$ away. We then query $\widehat q$ to multiple $\Scal_i(c, r, 0)$ with $i \in [\kappa]$, and return a correct result from any of the $\Scal_i(c, r, 0)$ as the result for query $q$. \Cref{fig:eps_net} is an illustration for such process, and shows that it solves $(c, r, \epsilon)$-MIPS problem. \Cref{alg:mips-adaptive} presents the pseudocode, which we will later refer as $\Mcal(c, r, \epsilon, \delta)$.

\begin{figure}[h]
    \includegraphics[clip, trim={250 150 250 150},width=\linewidth]{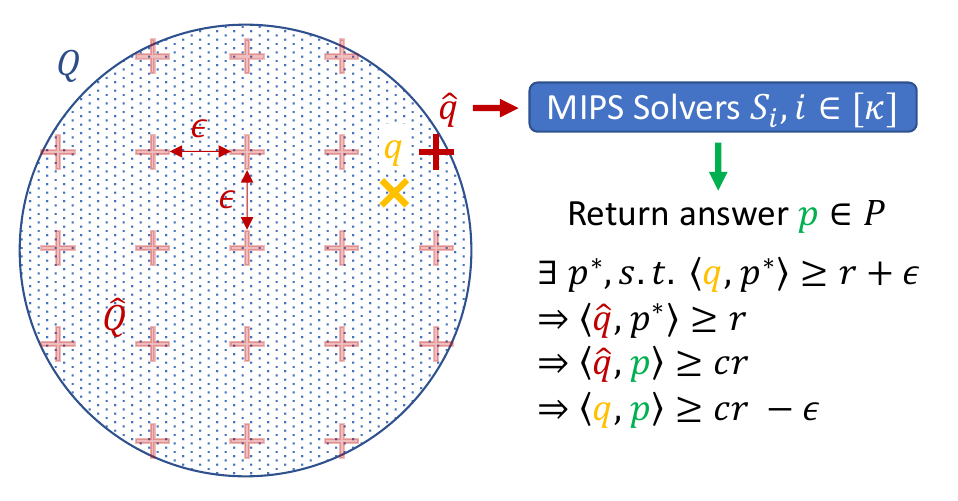}
    \caption{Illustration for \Cref{alg:mips-adaptive}. The blue circle represents the continuous set $Q$ which contains all possible queries, and $\widehat Q$ is an $\epsilon$-net in $Q$. For a query $q\in Q$, it is first rounded to $\widehat q \in \widehat Q$. Then the MIPS solvers $\Scal_{i}(c, r, 0), i\in [\kappa]$ are invoked to answer $\widehat q$. Suppose $\exists p^* \in P, s.t. \inner{q}{p^*} \ge r + \epsilon$ and a point $p\in P$ is returned by some $\Scal_i$. We have $\inner{q}{p} \ge cr - \epsilon$ as indicated by the figure. Thus \Cref{alg:mips-adaptive} solves $(c, r, \epsilon)$-MIPS.}
    \label{fig:eps_net}
\end{figure}

Our next result shows that with $\kappa\triangleq d\log \rbr{\frac{K d}{\epsilon\delta}}$ independent single query MIPS solvers $\Scal(c, r, 0)$, we can construct a MIPS solver $\Mcal(c, r, \epsilon, \delta)$ that gives correct answer for $(c, r, 0)$-MIPS problem for all queries in $\widehat Q$ with probability at least $\delta$. It therefore solves $(c, r, \epsilon)$-MIPS problem for an arbitrary sequence of queries (possibly adaptive) from $Q$. Coupled with \Cref{prop:non-adaptive}, our next result presents the sublinear in $K$ time complexity of \Cref{alg:mips-adaptive}.

\begin{algorithm}[t]
\centering 
\begin{algorithmic}[1]
\STATE \textbf{Preprocess}
\STATE \textbf{Input:} set of points $P \subseteq \RR^d$, parameter $(c, r, \epsilon)$ of the MIPS problem (see \Cref{def:mips}), desired failure probability bound $\delta$
\STATE Set $\kappa = {d\log \rbr{\frac{K d }{\epsilon\delta}}} $
\STATE Construct $\kappa$ non-adaptive $(c, r, 0)$-MIPS solvers $\Scal_i, i\in [\kappa]$ for the set of points $P$  (\Cref{prop:non-adaptive})
\\\hrulefill
\STATE \textbf{Add a new point $p$ to $P$}
\STATE \textbf{Input:} a new point $p \in \RR^d$
\STATE For all $\Scal_i, i \in [\kappa]$, add $p$ to $\Scal_i$
\\\hrulefill
\STATE \textbf{Query}
\STATE \textbf{Input:} query $q \in \RR^d$
\STATE Round the query $q$ to the nearest point $\widehat q$, whose coordinates are all multiples of $\frac{\epsilon}{d}$
\STATE Query all non-adaptive MIPS solvers $\cbr{\Scal_{i}}_{i \in [\kappa]}$ with $\widehat q$
\STATE \textbf{Return:} any point $p\in P$ returned by any of $\cbr{\Scal_{i}}_{i \in [\kappa]}$, otherwise return null
\end{algorithmic}
\caption{\textsc{Adaptive MIPS Solver $\Mcal(c, r, \epsilon, \delta)$} }
\label{alg:mips-adaptive}
\end{algorithm}

\begin{theorem}[Adaptive MIPS solver $\Mcal(c, r, \epsilon, \delta)$]\label{thm:near-linear-MIPS}
For a point set $P\in \RR^d$ with $K$ points, there exists a data structure $\Mcal(c, r, \epsilon, \delta)$ (\Cref{alg:mips-adaptive}) that sovles $(c, r, \epsilon)$-MIPS problem correctly for arbitrary (possibly adaptive) queries with at least $(1 - \delta)$ probability, for any $\delta \in (0,1)$. It has the following time complexity: \textbf{Preprocessing:} $\kappa K^{1+o(1)}$; \textbf{Add a Point to $P$:} $\kappa K^{o(1)}$; \textbf{Query:} $\kappa K^{\rho_q +o\rbr{\log^{-0.45}K}}$, where $\rho_q = \frac{4c'^2}{(1+c'^2)^2}$ and $c' = \sqrt{\frac{3 - cr}{3 -r}}$.
\end{theorem}

Here we illustrate how one can use multiple instances of $\Scal_i(c, r, 0)$ to answer all queries in $\widehat Q$ correctly. Note that the only failure case of $\Scal_i(c, r, 0)$ is when there exists $p^{*}\in P$ such that $\inner{q}{p^{*}} \ge r$ and it fails to return any $p$. This is because we can avoid spurious answer $p$ with a simple sanity check of $\inner{q}{p} \ge cr$. Therefore, outputting a point $p$ is an indicator of success, which allows for using multiple $\Scal_i$ to construct another one with a higher success probability.

%% file: problem_setup.tex
\section{Linear Bandits Problem Setup}\label{sec:bandits_setup}

We first introduce the \textbf{\textit{extremely large and slowly changing}} linear bandits problem setting. Let $\Acal$ be the set of all arms, where each of the arm $a\in\Acal$ has a feature vector $x_{a} \in \RR^d$.

The setting is called \textit{extremely large} as we focus on the regime where $\Acal$ is extremely large while time horizon $T$ is moderate (e.g., $T = \Theta(\log^\gamma K )$ for some constant $\gamma$).

The arm set $\Acal$ can \textit{change} in two ways: at each time step $t$ (1) there is a set of new arms $\Acal_{new}$ included into the arm set $\Acal$, but no deletions from $\Acal$; or more generally (2) there are some new arms $\Acal_{new}$ added, and some old arms in $\Acal$ deleted.
We use $K$ to denote the maximum size of $\Acal$, and our goal is to achieve per-step complexity sublinear in $K$.

Further, the arm set $\Acal$ changes \textit{slowly} in the sense that, at every time step, there is at most $C_{change}$ additions and deletions. For simplicity, we assume $C_{change}$ to be a constant in the rest of our paper. Our results, however, are valid for any $C_{change} = O(K^{\gamma})$ for any constant $\gamma < 1$.

At time step $t$, the online learner plays an arm $a_t \in \Acal$, and observes the reward $r_{t}$. We adopt the following commonly used assumptions:

\begin{assumption}[Linear Realizability]\label{as:1}
  $\exists\theta^{*}\in\RR^{d}$, such that $r_{t} = \inner{\theta^{*}}{x_{a_{t}}}+\eta_{t} $, where $\eta_{t}$ is a mean 0 noise.
\end{assumption}

\begin{assumption}[Subgaussian Noise]\label{as:2}
 The noise satisfies,
\begin{align*}
  \EE\sbr{\exp\rbr{\alpha\eta_{t}}\Bigm\vert \Fcal_{t-1}} \le \exp\rbr{\frac{\alpha^{2}}{2}}, \forall \alpha\in\RR, \forall t \in [T],
\end{align*}
with the filtration $\Fcal_{t - 1} = \sigma\rbr{a_{1}, r_{1}, \cdots, a_{t-1}, r_{t-1}, a_{t}}$.
\end{assumption}

\begin{assumption}[Bounded Parameters]\label{as:3}
  We assume that $\norm{\theta^{*}}_2 \le 1$ and $\norm{x_{a}}_2\le 1, \forall a\in\Acal$.
\end{assumption}

The regret is defined as $R(T) \coloneqq \sum_{t=1}^{T}x_{a_t^*}^\top\theta^{*} - x_{a_t}^\top\theta^{*}$, where $a_t^* \coloneqq \argmax_{a\in \Acal} x_a^\top \theta^{*}$ is the optimal arm at time step $t$. The goal of the online learner is to minimize the regret $R(T)$.

%% file: sublinear_elimination.tex
\section{Sublinear Time Elimination Algorithm}\label{sec:sublinear_elim}

In this section, we focus on an arm set $\Acal$ that keeps growing and no arm is deleted. We present an elimination-based algorithm that achieves sublinear per-step complexity. Intuitively, we adopt the MIPS solver to choose the arm with approximately the highest uncertainty in $o(K)$ time. The elimination-based algorithm is additionally faster in later stages, as many arms are eliminated.

We can estimate $\theta^*$ with an online ridge regression,
\begin{align}\label{eq:theta_estimation}
  \widehat \theta_{t+1} = \rbr{X_{1:t}^{\top}X_{1:t} + I}^{-1}X_{1:t}^{\top}Y_{1:t},
\end{align}
where $X_{1:t}$ is the matrix whose rows are $x_{a_{1}}^{\top}, \cdots, x_{a_{t}}^{\top}$ and $Y_{1:t} = (r_{1}, \cdots, r_{t})$. As established in \citep{abbasi2011improved}, for any $\delta \in (0, 1)$, with probability at least $(1 - \delta)$, for all $t \ge 1$, we have $
    \norm{\widehat\theta_t - \theta^*}_{V_t} \le \beta\rbr{\delta}$,
with $V_{t} = I + \sum_{s=1}^{t-1}x_{s}x_{s}^{\top}$ and $\beta(\delta) = 1 + \sqrt{2\log\rbr{\frac{1}{\delta}} + d\log\rbr{1+\frac{T}{d}}}$.

In the standard linear bandits setting, the arm set $\Acal$ is fixed and does not grow over time. An elimination-based algorithm typically selects the arm $a$ with the highest uncertainty, measured by $\norm{x_a}_{V_t^{-1}}$, and periodically eliminates the bad arms (i.e. the arms with $x_{a}^{\top}\widehat \theta_{t} + \beta(\delta)\norm{x_{a}}_{V_{t}^{-1}}$ smaller than $\underline{r}\triangleq\max_a x_{a}^{\top}\widehat \theta_{t} - \beta(\delta)\norm{x_{a}}_{V_{t}^{-1}}$). After elimination, any remaining arm $a$ costs at most $C\cdot \beta(\delta)\max\norm{x_a}_{V_t^{-1}}$ regret, whose summation over $T$ can be controlled by existing results.

Notice that the elimination requires a scan through all the arms. It is thus an $\Theta(K)$ time operation, which we do not hope to pay per-step. A common choice is to adopt stagewise elimination -- initializing $s = 1$ and eliminating when the uncertainty $\beta(\delta)\norm{x_a}_{V_t^{-1}}$ of all arms falls below $2^{-s}$, then increment $s$ by 1. The elimination therefore only happens $\log T$ times. In the next subsection, however, we show that such a simple strategy fails when $\Acal$ keeps growing.

\subsection{Efficient Elimination with Heap}

Elimination is necessary every time when $\Acal$ grows. As new arms $\Acal_{new}$ coming, the elimination threshold $\underline{r}\triangleq\max_a x_{a}^{\top}\widehat \theta_{t} - \beta(\delta)\norm{x_{a}}_{V_{t}^{-1}}$ might significantly increase. This typically happens when $\Acal_{new}$ contains an arm that is much better than the previously optimal arm. When $\underline{r}$ increases, some of arms that were not previously eliminated should be eliminated -- otherwise they might still be selected according to the criterion $\argmax_a \norm{x_a}_{V_t^{-1}}$ but incurring a regret much larger than $C \cdot \beta(\delta)\norm{x_a}_{V_t^{-1}}$, which possibly leads to an unbounded regret.

The necessity to eliminate arms according to the newly added arms calls for a more carefully designed data structure, which supports incremental elimination but avoids linear scanning through all arms $\Acal$.

Our solution is presented in \Cref{alg:enhanced_arm_set}, which partitions the arm set $\Acal$ into sets $\Psi_s$. The arms reside in $\Psi_s$ all have uncertainty $\beta(\delta)\norm{x_a}_{V_t^{-1}}$ smaller than $2^{-s}$.

More importantly, the arm set $\Psi_s$ is augmented with a min heap $\Hcal_s$, which stores arm $a$ indexed by $x_a^\top \widehat \theta + 2^{-s}$. Whenever a larger $\underline{r}$ appears, $\Psi_s$ can quickly compare the heap top $x_a^\top \widehat \theta + 2^{-s}$ with $\underline{r}$ and eliminates the arm $a$ as necessary. This avoids the linear scan for $\Acal$ when the elimination threshold $\underline{r}$ changes with the newly added arms $\Acal_{new}$.

An important implication is that after elimination (line 9 -- 14 of \Cref{alg:enhanced_arm_set}), playing an arm $a$ with the (approximately) largest uncertainty, the regret is again bounded by $C \cdot \beta(\delta)\norm{x_a}_{V_t^{-1}}$. Formally, at time step $t$, let $s_t$ be the minimum $s$ such that $\Psi_s$ is not empty, we have:
\begin{lemma}\label{lemma:bounded_regret}
For all $a \in \Psi_{s_t}$, $x_{a_t^*}^\top\theta^* - x_a^\top\theta^* \le 4\cdot 2^{-s_t}.$
\end{lemma}
The approximate MIPS query step (line 15 -- 18 of \Cref{alg:enhanced_arm_set}) can upper bound $2^{-s_t}$ by $16\cdot\beta(\delta)\norm{x_{a_t}}_{V_t^{-1}}$, up to some approximation error. It, therefore, retains the original regret guarantee (by following existing bounds on the summation of $\norm{x_{a_t}}_{V_t^{-1}}$ over $t$), without linearly scanning the arm set $\Acal$ at every step.

\begin{algorithm}[t]
\centering 
\begin{algorithmic}[1]
\STATE \textbf{Initialize}
\STATE \textbf{Input:} stage index $s$, parameters $d, \beta, \eta, \delta_\Psi$
\STATE Initialize an {adaptive MIPS solver} $\Mcal_{s}$ with $\rbr{\frac{1}{4}, \frac{2^{-2s-2}(1 - \eta^2)}{d^{2}\beta^{2}}, \frac{2^{-2s-2}\eta^{2}}{d^{2}\beta^{2}}, \delta_{\Psi}}$
\STATE Initialize an empty min heap $\Hcal_{s}$
\\\hrulefill
\STATE \textbf{Add}
\STATE \textbf{Input:} arm $a$, parameter estimate $\widehat\theta$
\STATE Add point $vec\rbr{{x_{a}x_{a}^{\top}}/{d}}$ to $\Mcal_{s}$
\STATE Push $(x_{a}^{\top}\widehat\theta + 2^{-s}, a)$ to heap $\Hcal_{s}$, using scalar $(x_{a}^{\top}\widehat\theta+2^{-s})$ for ordering
\\\hrulefill
\STATE \textbf{Eliminate}
\STATE \textbf{Input:} new elimination threshold $\underline r$
\WHILE {Heap $\Hcal_{s}$ top is smaller than $\underline r$}
\STATE $v, a = \Hcal_{s}\text{.pop()}$
\STATE Delete arm $a$ from $\Mcal_{s}$
\ENDWHILE
\\\hrulefill
\STATE \textbf{Query}
\STATE \textbf{Input:} $V \in \RR^{d\times d}$
\STATE Query $\Mcal_{s}$ with $vec(V / d)$, denote the $\Mcal_{s}$ output as $a$
\STATE \textbf{Return:} $a$ if $a$ is not null; otherwise return null
\end{algorithmic}
\caption{\textsc{Heap Augmented Arm Set $\Psi_{s}$} }
\label{alg:enhanced_arm_set}
\end{algorithm}

\subsection{Algorithm and Its Regret, Time Complexity}

Here we present the sublinear time elimination-based algorithm, and its regret and time complexity.

\begin{algorithm}[!t]
\centering 
\begin{algorithmic}[1]
\STATE \textbf{Input:} arm set $\Acal$, time horizon $T$, desired failure probability bound $\delta$, desired accuracy $\eta(T)$
\STATE Initialize $V_{1} = I, s = 1, \Acal_{1} = \Acal$
\STATE Set $\beta(\frac{\delta}{2}) = 1 + \sqrt{2\log \rbr{\frac{2}{\delta}} + d \log\rbr{1+\frac{T}{d}}}$
\STATE Set $s_{max}=\ceil{\log\frac{1}{8\eta(T)}}$, initialize $\Psi_{s}$ for $s \in [s_{max}]$ with $\rbr{s, d, \beta(\frac{\delta}{2}), \eta(T), \frac{\delta}{2 s_{max}}}$
\STATE Add all arms $a \in \Acal$ to $\Psi_{0}$
\FOR{$t = 1, 2, \cdots, T$}
\STATE \textcolor{blue}{/* Add new arms $\Acal_{new}$ */}
  \STATE For all $a \in \Acal_{new}$, add $a$ to $\Psi_{s}$ with $s = \min\rbr{\floor{-\log\rbr{\beta(\frac{\delta}{2})\norm{x_{a}}_{V_{t}^{-1}}}}, s_{max}}$
  \STATE Set $\underline  r' = \max_{a\in\Acal_{new}} \rbr{x_{a_{t}}^{\top}\widehat \theta_{t} - \beta(\frac{\delta}{2})\norm{x_{a}}_{V_{t}^{-1}}}$
  \IF {$\underline r' > \underline r$}
    \STATE Set $\underline r = \underline r'$. For all $s$, $\Psi_{s}$ eliminates arms with $\underline{r}$
  \ENDIF
  \STATE \textcolor{blue}{/* Choose an arm in $o(K)$ time */}
  \STATE Let $s_{t} = \argmin_s \abs{\Psi_s} > 0$
  \IF {$s_{t} = \ceil{\log \frac{1}{8\eta(T)}}$}
  \STATE Let $a_{t}$ be an random arm in $\Psi_{s_t}$
  \ELSE
  \STATE Let $a_{t}$ be the result of querying $\Psi_{s_t}$ with $V_{t}^{-1}$
  \WHILE {$a_{t}$ is null}
      \STATE Set $\underline  r' = \max_{a\in\Psi_{s_{t}}} \rbr{x_{a_{t}}^{\top}\widehat \theta_{t} - \beta(\frac{\delta}{2})\norm{x_{a}}_{V_{t}^{-1}}}$
    \STATE For all $a \in \Psi_{s_{t}}$, add $a$ to new set $\Psi_{s'}$ with $s' = \min\rbr{\floor{-\log \rbr{\beta(\frac{\delta}{2})\norm{x_{a}}_{V_{t}^{-1}}}}, s_{max}}$, and remove $a$ from $\Psi_{s_t}$
    \IF {$\underline r' > \underline r$}
      \STATE $\underline r = \underline r'$. For all $s$, $\Psi_{s}$ eliminate arms with $\underline{r}$
    \ENDIF
    \STATE Let $s_{t} = \argmin_s \abs{\Psi_s} > 0$
    \STATE Let $a_{t}$ be the result of querying $\Psi_{s_t}$ with $V_{t}^{-1}$
  \ENDWHILE
  \ENDIF
  \STATE Play arm $a_{t}$, observe reward $r_{t}$
  \STATE Update $V_{t+1} = V_{t} + x_{a_{t}}x_{a_{t}}^{\top}$
  \STATE Update $\widehat \theta_{t+1}$ according to \Cref{eq:theta_estimation}
\ENDFOR
\end{algorithmic}
\caption{\textsc{Sublinear Time Elimination}}
\label{alg:sublinear_elimination}
\end{algorithm}

The crux to achieve per-step $o(K)$ time complexity is twofold: (1) Selecting an arm that approximately has maximum uncertainty $\norm{x_a}_{V_t^{-1}} = \inner{V_{t}^{-1}}{x_{a}x_{a}^{\top}}$ is a MIPS problem. \Cref{alg:mips-adaptive} can solve it with sublinear time complexity; (2) The elimination (line 11 and line 23) uses \Cref{alg:enhanced_arm_set} as a sub-routine, and in total causes $K^{1+o(1)}$ complexity, which the algorithm does not need to pay per-step.

Running \Cref{alg:sublinear_elimination} for a linear bandits problem that satisfies Assumptions \ref{as:1} to \ref{as:3}, we have the following result for the regret and time complexity.

\begin{theorem}[Regret and time complexity of \Cref{alg:sublinear_elimination}, formal version see \Cref{thm:formal_sublinear_elimination}]\label{thm:sublinear_elimination}
For any $\delta \in (0, 1)$, with probability at least $1- \delta$, the regret is bounded by
  \begin{align*}
    R(T) = \widetilde O\rbr{d\sqrt{T} + \eta(T)\cdot T}
  ,\end{align*}
  with $\eta(T)$ controlling the approximate MIPS accuracy.
  
  The per-step time complexity is $K^{1 - \Theta(\frac{\eta(T)^4}{\log^2 T}) + o(\log^{-0.45}K)}$. The overall time complexity overhead (e.g., initialization) is $K^{1+o(1)}$.
\end{theorem}

$\eta(T)$ offers a trade-off between complexity and regret. The following corollaries show examples of choosing $\eta(T)$.

\begin{corollary}\label{coro:1}
  Given any $T$ that does not scale with $K$, one can choose $\eta(T) = \frac{1}{\sqrt{T}}$. The regret bound is  $\widetilde O(d\sqrt{T})$, while the per-step complexity is $K^{1 - \Theta(\frac{1}{T^2\log^2T})}$ for sufficiently large $K$. Note that this achieves per-step complexity sublinear in $K$ and retains the regret of $O(\sqrt{T})$.
\end{corollary}

\begin{corollary}\label{coro:2}
Consider the regime where $K$ is extremely large and $T = \Theta\rbr{\log^\gamma K}$ for some constant $\gamma$. Choosing $\eta(T) = T^{-\frac{0.1}{\gamma}}$, the regret bound is  $\widetilde O(T^{\frac{1}{2}} + T^{1 - \frac{0.1}{\gamma}})$, while the per-step complexity is $o(K)$. It shows that it is possible to achieve both sublinear regret and sublinear time complexity, for any large $K$ and moderate $T$.
\end{corollary}

One additional benefit of \Cref{alg:sublinear_elimination} is that the elimination typically removes many arms, which provides further speedup. Such speedup does not show up in the theoretical analysis as it depends on the distribution of arms. The acceleration brought by elimination is clearly presented in our empirical evaluation (\Cref{sec:experiment}).

Such additional speedup, however, comes with the price that the elimination-based algorithm can not handle deletions - as the remaining arms after elimination might get deleted from $\Acal$. In the next section, we present a sublinear time TS that allows for both additions and deletions.

%% file: accelerate_linTS.tex
\section{Sublinear Time TS-based Algorithm}\label{sec:acc_linTS}

In this section, we present a Thompson Sampling (TS) based algorithm with sublinear per-step time complexity. It works for the general arm set changing, where arms can be added to or deleted from $\Acal$. The TS-based algorithm also avoids paying the overhead for elimination (as required by \Cref{alg:sublinear_elimination}), and therefore after initialization, the time complexity for every time step is sublinear in $K$.

\subsection{Algorithm and its Regret, Time Complexity}

The linear TS algorithm \citep{abeille2017linear} maintains the estimation $\widehat \theta_{t}$ as \Cref{eq:theta_estimation}. At each time step $t$, a random $\widetilde \theta_{t}$ is constructed as $\widetilde \theta_{t} = \widehat \theta_{t} + \beta(\frac{\delta}{4T}) V_{t}^{-1/2}\xi_{t}$, with $\xi_{t}$ drawn from distribution $\Dcal^{TS}$, which satisfies \textit{concentration} and \textit{anti-concentration} properties (see \Cref{def:D-TS} in Appendix). For instance, $\Dcal^{TS}$ can simply be a spherical Gaussian distribution.

After $\widetilde \theta_{t}$ is constructed, the standard linear TS algorithm chooses the arm $a$ that maximizes $x_{a}^{\top}\widetilde \theta_{t}$. \Cref{alg:mips-adaptive} can be naturally applied to solve this MIPS for arm selection. See \Cref{alg:accelerated-linTS} for detail.

Notice that \Cref{alg:accelerated-linTS} assumes that the largest reward expectation is non-negative, as it is more commonly seen (e.g., when the reward corresponds to clicks, purchases, or ratings). When the largest reward expectation is negative, we propose the following extension: we can transform arm's feature $x$ to $\sbr{\frac{x}{\sqrt{2}}, \frac{\sqrt 2}{2}}$, observed reward $r_t$ to be $\frac{r_t}{2} + \frac{1}{2}$. The corresponding $\theta^*$ becomes $\sbr{\frac{\theta^*}{\sqrt 2}, \frac{\sqrt 2}{2}}$. In this way, the algorithm sees an environment with the largest reward expectation being positive and properly makes arm selection, while the true environment allows the largest reward expectation to be negative.

\begin{algorithm}[t]
\centering 
\begin{algorithmic}[1]
\STATE \textbf{Input:} arm set $\Acal$, time horizon $T$, desired failure probability bound $\delta$, desired accuracy $\eta(T)$
\STATE Set $\beta(\frac{\delta}{4T}) = 1 + \sqrt{2\log \rbr{\frac{4T}{\delta}} + d \log\rbr{\frac{d+T}{d}}}, V_{1} = I$
\STATE Preprocess $x_a, \forall a \in \Acal$ with \Cref{alg:mips-adaptive} with $\ceil{\frac{d}{\eta(T)}}$ independent copies. For the $i$-th copy $\Mcal_i$, use parameter $\rbr{1 - \frac{1}{i+1}, \frac{i \cdot \eta(T)}{d}, \frac{\eta(T)}{d}, \frac{\delta\cdot\eta(T)}{2d}}$
\STATE Add all arms $a\in \Acal$ to all $\Mcal_i$
\FOR{$t = 1, 2, \cdots, T$}
\STATE Add or delete the changing arms $a$ for all $\Mcal_s$ with $s \le \ceil{d/\eta(T)}$
\STATE Sample $\xi_{t} \sim \Dcal^{TS}$
\STATE Compute $\widetilde \theta_{t} = \widehat \theta_{t} + \beta(\frac{\delta}{4T}) V_{t}^{-1/2}\xi_{t}$
\STATE Query \Cref{alg:mips-adaptive} with ${\widetilde \theta_{t}}/{\norm{\widetilde\theta_t}}$ and different $m$, set $a_{t}$ to be the non-null result with largest $m$
\STATE Play arm $a_{t}$ and observe reward $r_{t}$
\STATE Update $V_{t+1} = V_{t} + x_{t}x_{t}^{\top}$
\STATE Update $\widehat \theta_{t+1}$ according to \Cref{eq:theta_estimation}
\ENDFOR
\end{algorithmic}
\caption{\textsc{Sublinear Time Thompson Sampling} }
\label{alg:accelerated-linTS}
\end{algorithm}

Under \Cref{as:1} to \ref{as:3}, we can characterize the regret and time complexity of \Cref{alg:accelerated-linTS} as following:

\begin{theorem}[Regret and time complexity of \Cref{alg:accelerated-linTS}, formal version see \Cref{thm:formal_acc-linTS}]\label{thm:acc-linTS}
  For any $\delta \in (0, 1)$, with probability at least $1-\delta$, the regret is bounded by
  \begin{align*}
  R(T) = \widetilde O\rbr{d^{3/2}\sqrt{T} + \eta(T)\cdot T},\end{align*}
  with $\eta(T)$ controlling the approximate MIPS accuracy.

  The per-step time complexity is $K^{1 - \Theta(\eta(T)^2) + o\rbr{\log^{-0.45} K}}$. The time complexity of the data structure maintenance (line 4) is $K^{1+o(1)}$ which is paid once at initialization.
\end{theorem}

$\eta(T)$ offers a trade-off between complexity and regret. The following corollaries show examples of choosing $\eta(T)$.

\begin{corollary}\label{coro:3}
For any $T$ not scaling with $K$, one can choose $\eta(T) = \frac{1}{\sqrt{T}}$. The regret bound is  $\widetilde O(d^{\frac{3}{2}}\sqrt{T})$, and the per-step complexity is $K^{1 - \Theta(\frac{1}{T})}$ for sufficiently large $K$. Note that this retains the regret of the linear TS algorithm and achieves per-step complexity sublinear in $K$.
\end{corollary}

\begin{corollary}\label{coro:4}
Consider the regime where $K$ is extremely large and $T = \Theta\rbr{\log^\gamma K}$ for some constant $\gamma$. Choosing $\eta(T) = T^{-\frac{0.2}{\gamma}}$, the regret bound is  $\widetilde O(T^{\frac{1}{2}} + T^{1 - \frac{0.2}{\gamma}})$, while the per-step complexity is $o(K)$.
\end{corollary}

%% file: experiments.tex
\section{Experiments}\label{sec:experiment}

\begin{table*}[!t]
\centering
\begin{tabular}{@{}cc|cc|cc@{}}
\toprule
                                                        &          & Linear Elim & Sub-Elim, shortlist 30 & Linear TS & Sub-TS, shortlist 30 \\ \midrule
\multicolumn{1}{c|}{\multirow{3}{*}{$K=5,000$}}   & Regret   &          $3866 \pm 195$   &          $3758 \pm 190$             &     $582 \pm 54$     &     $605 \pm 59$    \\
\multicolumn{1}{c|}{}                                   & Time (s) &      $11.74$       &   $2.22~(1.99)$      &      $30.08$    &   $19.41~(19.29)$  \\
\multicolumn{1}{c|}{}                                   & Speedup &       $\times 1$      &    $\bm{\times 5.28~(\times 5.89)}$   &  $\times 1$  &   $\bm{\times 1.55~(\times 1.56)}$  \\ \midrule
\multicolumn{1}{c|}{\multirow{3}{*}{$K=100,000$}} & Regret   &  $4804 \pm 146$   &      $4701 \pm 150$     &  $721 \pm 92$   &  $734 \pm89$   \\
\multicolumn{1}{c|}{}                                   & Time (s) &  $221.19$   &    $59.40~(3.04)$    &   $280.78$    &   $32.81~(29.10)$      \\
\multicolumn{1}{c|}{}                                   & Speedup &     $\times 1$        &    $\bm{\times 3.72~(\times 72.76)}$    &   $\times 1$    &  $\bm{\times 8.56~(\times 9.65)}$     \\ \bottomrule
\end{tabular}
\vspace{+0.5em}
\caption{\textbf{Synthetic Experiment - Impact of Different $K$.} ``Linear Elim" and ``Linear TS" are baselines. ``Sub-Elim" and ``Sub-TS" are \Cref{alg:sublinear_elimination,alg:accelerated-linTS}, with the shortlist being 30. ``Regret" corresponds to the cumulative regret of 20,000 steps, with mean and standard deviation for 10 independent runs. The reported ``Time" corresponds to the overall running time of 20,000 steps, averaged over 10 independent runs. The running time excluding preprocessing is reported in the bracket. The ``Speedup" is the relative speedup compared with the corresponding baselines. The results demonstrate that Sub-Elim and Sub-TS can deliver significant speedup (e.g., a 72.76 times speedup, excluding preprocessing) especially when the number of arms $K$ is large while obtaining a similar regret as the linear time baselines.}
\label{table:different_K}
\vspace{-1em}
\end{table*}

\begin{table}[!t]
\resizebox{0.48\textwidth}{!}{
\begin{tabular}{@{}c|ccc@{}}
\toprule
Algorithm & Linear Elim & Sub-Elim, shortlist 10 & Sub-Elim, shortlist 100 \\ \midrule
Regret    & $4803 \pm 146$    & $4691 \pm 133$   &  $4837 \pm 143$      \\
Time(s)   & $221.19$ &  $59.07~(2.85)$  &  $59.84~(3.91)$   \\
Speedup  & $\times 1$   &  $\bf{\times 3.74~(\times77.61)}$  &  $\bm{\times 3.69~(\times 56.57)}$ \\ \bottomrule
\end{tabular}}
\resizebox{0.48\textwidth}{!}{
\begin{tabular}{@{}c|ccc@{}}
\toprule
Algorithm & Linear TS & Sub-TS, shortlist 10 & Sub-TS, shortlist 100 \\ \midrule
Regret    & $721\pm 92$ &   $736 \pm 89$   &  $721 \pm 92$     \\
Time(s)   & $280.78$   &    $31.44~(27.75)$      &  $36.35~(32.64)$     \\
Speedup  & $\times 1$    &  $\bm{\times 8.93~(\times 10.12)}$     &  $\bm{\times 7.72~(\times 8.60)}$   \\ \bottomrule
\end{tabular}}
\caption{\textbf{Synthetic Experiment - Impact of Approximation Precision.} The algorithms and ``Regret", ``Time" and ``Speedup" are defined the same as in \Cref{table:different_K}. Combining with the ``shortlist 30" results in \Cref{table:different_K}, it shows that a lager shortlist size  $p$ (corresponds to a smaller $\eta(T)$ in \Cref{alg:sublinear_elimination,alg:accelerated-linTS}) leads to longer running time. In our evaluated settings, all different shortlist sizes $p$ are large enough to keep regret similar to the linear time baselines.}
\label{table:different_shortlist}
\vspace{-1.5em}
\end{table}

In this section, we empirically evaluate the performance of our proposed algorithms in a synthetic environment and a real-world problem on movie recommendation.

We adopt the following algorithms for evaluation:
\begin{itemize}[leftmargin=*, itemsep=0.5ex]
    \item \textbf{Sublinear Time Elimination (Sub-Elim)}: We implement \Cref{alg:sublinear_elimination} and use HNSW algorithm \citep{malkov2018efficient} as the MIPS solver in \Cref{alg:enhanced_arm_set}.
     \item \textbf{Sublinear Time Thompson Sampling (Sub-TS)}: We implement \Cref{alg:accelerated-linTS} with HNSW as the MIPS solver.
    \item \textbf{Baselines}: We implement the linear time version of \Cref{alg:sublinear_elimination,alg:accelerated-linTS}, where the MIPS step is solved by the standard linear scan through all the arms. Such baselines allow us to evaluate the performance and acceleration brought by adopting an approximate MIPS solver.
\end{itemize}

LSH is not used for our implementation as there is currently no efficient LSH implementation that supports deletions. Note that this is purely an engineering issue - there exist LSH constructions that theoretically support efficient deletions \citep{andoni2017optimal}.

To control the tradeoff between MIPS accuracy and time complexity, we construct the MIPS solver in the following way: We first use the HNSW algorithm to retrieve a shortlist of $p$ arms, then linearly scan the retrieved $p$ arms for the one with the largest inner product. A larger $p$ gives higher accuracy but slower speed. We take $p$ from {$\cbr{10, 30, 100}$} for our experiments. The choices of different $p$ can be viewed as different $\eta(T)$ for \Cref{alg:sublinear_elimination,alg:accelerated-linTS}.

\paragraph{Synthetic Experiment} For the synthetic experiment, we first randomly generated a $16$-dimensional vector $\theta^*$ from a Gaussian distribution $\Ncal(0, \Ib_{16})$. The arms $\Acal$ are generated from the same distribution. The reward noise is unit Gaussian. Further, over the time horizon $T = 20,000$, a batch of $C_{change} = 2$ arms are generated and included into the arm set $\Acal$ every $20$ steps. The final arm set size is $K$.

Our first result (\Cref{table:different_K}) demonstrates the efficiency of Sub-Elim and Sub-TS with different numbers of arms $K$. In particular, when the number of arms $K$ is large, the Sub-Elim is able to deliver a 72.76 times speedup (excluding the preprocessing time) while retaining the regret of the linear time implementation.

\begin{table}[!t]
\resizebox{0.48\textwidth}{!}{
\begin{tabular}{@{}c|ccc@{}}
\toprule
Algorithm & Linear Elim & Sub-Elim, shortlist 30 & Sub-Elim, shortlist 100 \\ \midrule
Regret    & $3847 \pm 212$ &$3795 \pm 206$   &  $3806 \pm 206$       \\
Time(s)   & $29.55$ &   $4.22~(3.47)$     & $4.83~(4.09)$    \\
Speedup  & $\times 1$ &  $\bm{\times 7.00~(\times 8.52)}$   &  $\bm{\times 6.12~(\times 7.22)}$  \\ \bottomrule
\end{tabular}}
\resizebox{0.48\textwidth}{!}{
\begin{tabular}{@{}c|ccc@{}}
\toprule
Algorithm & Linear TS & Sub-TS, shortlist 30 & Sub-TS, shortlist 100 \\ \midrule
Regret    &  $1193 \pm 66$    &  $1177 \pm 66$  &   $1202 \pm 68$    \\
Time(s)   & $29.83$ &   $19.59~(19.38)$    &    $20.63~(20.41)$                     \\
Speedup  & $\times 1$   &   $\bm{\times 1.52~(\times 1.54)}$   &  $\bm{\times 1.45~(\times 1.46)}$ \\ \bottomrule
\end{tabular}}
\caption{\textbf{Movie Recommendation - Running time and regret.} ``Regret" corresponds to the cumulative regret of 20,000 recommendations, with mean and standard deviation for 300 users. ``Time" is the total running time of making 20,000 recommendations, averaged over 300 users. The time excluding preprocessing is reported in the bracket. The results show that ``Sub-Elim" is more than 7 times faster; and ``Sub-TS" can reduce 30\% of the baseline's running time. Both have similar regret as baselines.}
\label{table:movie_exp}
\vspace{-1em}
\end{table}

We further evaluate the impact of different choices of shortlist size $p$ (i.e., a larger $p$ corresponding to a more accurate approximate MIPS solver) and the results are presented in \Cref{table:different_shortlist}. Moreover, we evaluate our algorithms with $C_{change}\in\cbr{2, 10, 50}$ and show that all our algorithms can deliver stable speedup in the evaluated settings. We also test \Cref{alg:accelerated-linTS} when there are both additions and deletions, which demonstrates a speedup and comparable regret as baselines. The results are deferred to \Cref{apdx:exp}.

\paragraph{Movie Recommendation}

The testing environment is derived from a popular recommendation dataset: Movielens-1M \citep{harper2015movielens}. The dataset contains over 1 million ratings of 3,952 movies by more than 6,000 users.

The environment construction is similar to \citep{qin2014contextual}. We preserve the ratings of 300 users (each with more than 100 ratings) for testing. With the ratings of more than 5,700 remaining users, we create a 16-dimensional feature for each of the movies by matrix factorization. The movies' features are used as arms' features ($\xb_t(i)$).

The algorithm starts with $1,952$ movies, and interacts with the user for $20,000$ times (i.e., time horizon $T = 20,000$). 2 new movies are included for every 20 steps, which in the end leads to all 3,952 movies. In each time step, the regret is 1 if the recommended movie has a rating smaller than 4 or no rating, and otherwise, the regret is 0.

The average regret (and standard deviation) and the running time are reported in \Cref{table:movie_exp}.  Our empirical results demonstrate the acceleration and great empirical performance of the proposed sublinear time algorithms.

\section*{Acknowledgment}
This work was partially supported by NSF grants 1564000, 1934932, 2019844, and 2107037, and the Machine Learning Lab at UT Austin.

%% file: apdx_proof.tex
\section{Proof for \Cref{sec:mips}}
\subsection{Proof of \Cref{prop:non-adaptive}}

\citet{andoni2017optimal} proposed a data structure that solves the approximate nearest neighbor (ANN) search problem. We therefore first present a transformation, which converts a MIPS problem into the nearest neighbor search problem. The transformation is proposed in \cite{bachrach2014speeding}.

A $(c', r')$-approximate nearest neighbor search problem aims to find $p' \in P' \subseteq \RR^{d+3}$ for a query $q'\in \RR^{d+3}$ such that $\norm{p' - q'}_2 \le c' r'$, if there exists $\widetilde p \in P'$ such that $\norm{\widetilde p - q'}_2 \le r'$. Recall that for MIPS problem, we have $\norm{p}_2 \le 1, \forall p \in P$ and $\norm{q}_2 \le 1$ by \Cref{def:mips}. We take the following transformation
  \begin{align*}
      & p' = \sbr{\frac{1}{2}; \frac{p}{2}; \sqrt{\frac{3 - \norm{p}^2_2}{4}}; 0} \in \RR^{d+3}, \forall p \in P;\\
      & q' = \sbr{\frac{1}{2}; \frac{q}{2}; 0; \sqrt{\frac{3 - \norm{{q}}_2^2}{4}}} \in \RR^{d+3}.
  \end{align*}
Let $P' = \cbr{p'~|~\forall p \in P}$. Then for any point $p'\in P'$ and any query $q'$, we have
\begin{align*}
    \norm{p' - q'}_2^2 & = \norm{p'}_2^2 + \norm{q'}_2^2 - 2\inner{p'}{q'} \\
    & = \frac{3 - \inner{p}{q}}{2}.
\end{align*}
Therefore the original $(c, r, 0)$-MIPS is equivalent to $(c', r')$-ANN with $c' = \sqrt{3 -  cr} / \sqrt{3 - r}$ and $r' = \sqrt{\frac{3-r}{2}}$. For $c \in [0, 1)$ and $r \in (0, 1]$, we have $r' \in [1, \sqrt{3/2}), c'r' \in (1, \sqrt{3/2}]$ and $c' \in (1, \sqrt{3/2}]$.

\citet{andoni2017optimal} constructed a data structure that solves $(c', r')$-ANN with constant success probability. It has $K^{1 + o(1)}$ preprocessing time complexity, $K^{\rho_q + o(1)}$ query time complexity and $K^{o(1)}$ time complexity for adding a new point to $P$. The rest of our proof follows the same procedure as \citep{andoni2017optimal}, but aims to give a more explicit characterization for the $o(1)$ term in the query time complexity, which turns out to be $o\rbr{\log^{-0.45}K}$. This more explicit form is useful when $\rho_q$ is very close to 1 (i.e., $\rho_q = 1 - o(1)$).

\textbf{Short description of the data structure construction.} 

The proposed data structure stores all data points $P$ in a tree, with depth $M$ and branching factor at most $B$.

During preprocessing, each tree node $n$ draws a unit norm vector $u_n$ uniformly at random. Each point $p \in P$ will traverse down the tree from the root, the point $p$ will descend through a node $n$ if the inner product $\inner{p}{u_n} \ge \eta_u$, where $\eta_u$ is a scalar parameter that is shared for the entire tree (i.e., all nodes use the same $\eta_u$). The point can descend through multiple nodes at the same level, and will possibly reach multiple leave nodes in the end. The leave nodes will store all the points $p$ that reached it during preprocessing.

During query time, a query $q$ will also descend from the root, and go down through the nodes with $\inner{q}{u_n} \ge \eta_q$. Similar to a point $p$ in the preprocessing stage, the query $q$ will possibly reach multiple leave nodes in the end. It will then linearly scan through all the points $p$ stored in the corresponding leave nodes. It will stop scanning and return the first point $p$ that solves the $(c', r')$-ANN problem.

We will omit much detail of the proof but highlight the difference. One can check the full proof in Section 3.3.3 of \citep{andoni2017optimal}. Define $F(\eta) \coloneqq \PP_{z\sim N(0, 1)^{d}}\sbr{\inner{z}{u} \ge \eta}$, where $u$ is an arbitrary point on the unit sphere. Define $G(s, \eta, \sigma) = \PP_{z\sim N(0, 1)^{d}}\sbr{\inner{z}{u} \ge \eta \text{ and } \inner{z}{v} \ge \sigma}$, where $u, v$ are two points on the unit sphere with $\norm{u - v}_{2} = s$. 

\textbf{Preprocessing time complexity}

We now prove the preprocessing time complexity. Notice that $F(\eta_{u})$ is the probability of one point $p\in P$ descends from a node to a child node. 
\begin{lemma}
  The data structure construction has the following complexity in expectation:
  \begin{align*}
    & \textbf{Time : } K^{1+o(1)}\cdot B \cdot \sum_{i=0}^{M}\cdot \rbr{B\cdot F(\eta_{u})}^{i}. \\
    & \textbf{Space : } K^{1+o(1)}\cdot M\cdot \rbr{B\cdot F(\eta_{u})}^{M}
  .\end{align*}
\end{lemma}
\begin{proof}
  The analysis for space complexity is presented in Lemma 3.7 in \citep{andoni2017optimal}. We show the time complexity analysis in a similar way.

  In the preprocessing of $P$, a point $p\in P$ in expectation descends to $B\cdot F(\eta_{u})$ nodes from one node. Thus the expected number of points at depth-$i$ is $K\cdot\rbr{B\cdot F(\eta_{u})}^{i}$. Each point in at one node incurs a time complexity of $B\cdot K^{o(1)}$. In our regime of interest, $K$ is extremely large and we treat the dimension $d$ as $K^{o(1)}$.

  There is an over-estimation at the depth-$M$ node. Since there is no further branching in such nodes, each point at one depth-$M$ node only incurs $K^{o(1)}$ time complexity, instead of $B\cdot K^{o(1)}$. This over-estimation does not hurt further analysis.
\end{proof}

Next, we show that the preprocessing time complexity is the same as the space complexity.
\begin{lemma}
  Both time and space compelxity of data structure construction are $K^{1 + o(1)}\cdot \rbr{B\cdot F(\eta_{u})}^{M}$.
\end{lemma}
\begin{proof}
  As suggested in \cite{andoni2017optimal}, we set $M = \sqrt{\log K}$, which immediately implies $K^{1 + o(1)}\cdot \rbr{B\cdot F(\eta_{u})}^{M}$ space complexity.

  For time complexity, we have
  \begin{align*}
    \sum_{i=0}^{M}(B\cdot F(\eta_{u}))^{i} = O(1)(B\cdot F(\eta_{u}))^{M}
  .\end{align*}
  This follows from $F(\eta_{u}) \ge G(r, \eta_{u}, \eta_{q})$, and thus $B\cdot F(\eta_{u}) \ge B \cdot G(r, \eta_{u}, \eta_{q})$, where in the analysis of \citep{andoni2017optimal} it sets $B \cdot G(r, \eta_u, \eta_q) \ge 100$. In the proof of the optimal $\rho_{u}, \rho_{q}$ trade-off by \citet{andoni2017optimal}, it showed that when $\rho_u = 0$ (which is the setting we adopted), the specified $M, B$ leads to
  \begin{align*}
    B^{M} = K^{c_0}
  .\end{align*}
  The $c_0$ is a constant not depending on $K$. For $M = \sqrt{\log K}$, we have $B = K^{\frac{c'}{M}} = K^{o(1)}$. Putting these together, we have the time complexity to be $K^{1 + o(1)}(B\cdot F(\eta_{u}))^{M}$.
\end{proof}

With the choice of $\eta_{u} = 0$, the complexity $K^{1 + o(1)}(B\cdot F(\eta_{u}))^{M}$ is $K^{1 + o(1)}$ (see detailed proof in Section 3.3.3 \citep{andoni2017optimal}). It therefore achieves the $K^{1 + o(1)}$ time complexity.

\textbf{Query time complexity}

Here we show the $K^{\rho_q + o\rbr{\log^{-0.45}K}}$ query complexity extended from \citep{andoni2017optimal}, where the query complexity was presented as $K^{\rho_q + o(1)}$. Note that this is not an improvement over the original analysis. We are only more explicit about the $o(1)$ term which is necessary for our case.

During query time, the query $q$ recursively descends from the root, with each descending happening with probability $F(\eta_q)$. According to the Lemma 3.8 of \citep{andoni2017optimal}, the query time complexity is
\begin{align*}
    d \cdot B \cdot (B\cdot F(\eta_q))^M + K\cdot d \cdot (B\cdot G(c'r', \eta_u, \eta_q))^M,
\end{align*}
where $F(\eta_q)$ denotes the probability of the query descending from one node to one of its child node, and $G(c'r', \eta_u, \eta_q)$ denotes the probability of a query and a qualifying point (i.e., distance smaller than $c'r'$) in $P$ both descending to a child node. In the proof of query time complexity, \citet{andoni2017optimal} take that $F(\eta_u)^M = K^{-\sigma}$ and $F(\eta_q)^M = K^{-\tau}$.

We first present a stronger version of Lemma 3.1 in \cite{andoni2017optimal},
\begin{align*}
    F(\eta_q) = e^{-(1 + o(\eta_q^{-9/5}))\cdot\frac{\eta_q^2}{2}},
\end{align*} where the original Lemma 3.1 states $F(\eta_q) = e^{-(1 + o(1))\cdot\frac{\eta_q^2}{2}}$. This stronger version of Lemma 3.1 follows immediately from a tight Gaussian tail bound. As $\eta_q \rightarrow \infty$, we have
\begin{align*}
    \rbr{\frac{1}{\eta_q} - \frac{1}{\eta_q^3}}\frac{e^{-\eta_q^2 / 2}}{\sqrt{2\pi}} \le F(\eta_q) \le \frac{1}{\eta_q}\cdot\frac{e^{-\eta_q^2/2}}{\sqrt{2\pi}}.
\end{align*}

Follow the same analysis as in \cite{andoni2017optimal}, the stronger version of Lemma 3.1 implies a stronger version of Lemma 3.2 of \cite{andoni2017optimal},
\begin{align*}
    G(c'r',\eta_u,\eta_q) &= e^{-(1+o(\eta_q^{-9/5}))\cdot \frac{\eta_u^2 + \eta_q^2 - 2\alpha(c'r')\eta_u\eta_q}{2\beta^2(c'r')}},
\end{align*}
where $\alpha(s) = 1 - \frac{s^2}{2}$ is the cosine of the angle between two points on a unit Euclidean sphere with distance $s$ between them, and $\beta(s) = \sqrt{1 - \alpha^2(s)}$ is the sine of the same angle. Note that the original Lemma 3.2 is $G(s,\eta,\sigma) = e^{-(1+o(1))\cdot \frac{\eta^2 + \sigma^2 - 2\alpha(s)\eta\sigma}{2\beta^2(s)}}$. By requiring that $F(\eta_q)^M = K\cdot G(c'r', \eta_u, \eta_q)^M$, as $\eta_q \rightarrow \infty$, we have
\begin{align*}
     {\frac{\sigma + \tau - 2\alpha(c'r')\cdot\sqrt{\sigma\tau}}{\beta^2(c'r')} - 1} = \rbr{1 + o\rbr{\eta_q^{-9/5}}}\tau.
\end{align*}
As suggested in \cite{andoni2017optimal}, to have $\eta_u = 0$, we should set $\sqrt{\tau} = \frac{\alpha(r')\beta(c'r')}{1 - \alpha(r')\alpha(c'r')}$. With the transformation (from MIPS to ANN) proposed previously, we have $r' \in [1, \sqrt{3/2}), c'r' \in (1, \sqrt{3/2}]$. Therefore $\tau$ is bounded by constants as $\tau \in [0.06, 0.34]$, and we have
\begin{align*}
    \tau = \frac{\sigma + \tau - 2\alpha(c'r')\cdot\sqrt{\sigma\tau}}{\beta^2(c'r')} - 1 + o\rbr{\eta_q^{-9/5}}.
\end{align*}
Notice that with $F(\eta_q)^M = K^{-\tau}$ and $\tau$ bounded by constants, we have $\eta_q = \Omega(\log^{1/4} K)$ and therefore,
\begin{align}\label{eq:2}
    \tau = \frac{\sigma + \tau - 2\alpha(c'r')\cdot\sqrt{\sigma\tau}}{\beta^2(c'r')} - 1 + o\rbr{\log^{-0.45}K}.
\end{align}
In the original analysis by \cite{andoni2017optimal}, the result was
\begin{align*}
    \tau = \frac{\sigma + \tau - 2\alpha(c'r')\cdot\sqrt{\sigma\tau}}{\beta^2(c'r')} - 1 + o\rbr{1}.
\end{align*}
Therefore from \Cref{eq:2}, we have that, up to $o\rbr{\log^{-0.45}K}$ terms,
\begin{align*}
    \sqrt{\sigma} = \alpha(c'r')\sqrt{\tau} + \beta(c'r')
\end{align*}
Further, with $r' \in [1, \sqrt{3/2}), c'r' \in (1, \sqrt{3/2}]$, we have the following term also bounded by constants:
\begin{align*}
    \frac{\sigma + \tau - \alpha(r')\sqrt{\sigma \tau}}{\beta(r')^2} \in [0.97, 1.34].
\end{align*}

The rest of analysis follows the same as Section 3.3.3 in \citep{andoni2017optimal}, with all $o(1)$ replaced by $o(\log^{-0.45}K)$. As a result, the query time is $K^{\rho_q + o\rbr{\log^{-0.45}K}}$.

\textbf{Time complexity for adding a new point to $P$}

Adding a point to the data structure takes $K^{o(1)}\cdot B \cdot \sum_{i=0}^{M}\cdot \rbr{B\cdot F(\eta_{u})}^{i}$ time. We have proven in the \textbf{Preprocessing time complexity} that it is $K^{o(1)}$ under the choice of $\eta_u = 0$. Therefore the complexity of adding a new point is $K^{o(1)}$.

\subsection{Proof of \Cref{thm:near-linear-MIPS}}
\begin{proof}
  Denote $Q$ to be the unit $l_{2}$ ball in $\RR^{d}$ centered at 0. We have $q_{t}\in Q, \forall t\in [T]$. We can discretize $Q$ into lattice $\widehat Q$ with precision $\frac{\epsilon}{d}$. Note that every point in $\widehat Q$ has all its coordinates being multiples of $\frac{\epsilon}{d}$. We can then bound the size of $\widehat Q$ to be $\abs{\widehat Q} \le \rbr{\frac{2d}{\epsilon}}^{d}$.

  The probability of all $\kappa$ copies of $\Scal(c, r, 0)$ fail for any $\widehat q\in \widehat Q$ and any point $p \in p$ is
  \begin{align*}
    & \PP\rbr{\text{$\exists \widehat q\in\widehat Q, p\in p$ \text{ s.t. all } $\Scal(c, r, 0, \delta)$ fail on $p, \widehat q$}}
    \le K \rbr{\frac{2d}{\epsilon}}^{d}0.1^{\kappa} \le \delta
  .\end{align*}
  the last inequality follows from $\kappa = d\log \rbr{\frac{K d}{\epsilon\delta}} \ge \log\rbr{\frac{1}{K}\rbr{\frac{2d}{\epsilon}}^{-d}{\delta}}/\log\rbr{0.1}$.
  
  For any query $q \in Q$, rounding it to the nearest point $\widehat q \in \widehat Q$, it induces $\epsilon$ additive error for inner product (recall that $\norm{p} \le 1, \forall p \in P$). Thus, for arbitrary query sequence from $Q$, running $\kappa$ copies of $\Scal(c,r, 0)$ solves $(c,r,\epsilon)$-MIPS problem successfully for all the queries with probability at least $1 - \delta$. This completes the proof.
\end{proof}

\section{Proof for \Cref{sec:sublinear_elim,sec:acc_linTS}}

\subsection{Definition of TS distribution}
\begin{definition}[\citet{abeille2017linear}]\label{def:D-TS}
  $\Dcal^{TS}$ is a multivariate distribution on $\RR^{d}$ absolutely continuous with respect to the Lebesgue measure which satisfies the following properties:
  \begin{enumerate}
    \item (anti-concentration) there exists a strictly positive probability $p$ such that for any $u\in\RR^{d}$ with $\norm{u}_2 = 1$,
      \begin{align*}
        \PP_{\xi\sim\Dcal^{TS}}\rbr{u^{\top}\xi \ge 1} \ge p
      .\end{align*}
    \item (concentration) there exists $b, b'$ positive constants such that $\forall \delta\in(0, 1)$
      \begin{align*}
        \PP_{\xi\sim\Dcal^{TS}}\rbr{\norm{\xi} \le \sqrt{bd\log\frac{b'd}{\delta}}} \ge 1-\delta
      .\end{align*}
  \end{enumerate}
\end{definition}

\subsection{Technical Lemma}
We first present 2 previously established supporting lemmas on bounding $\norm{\widehat \theta_{t} - \theta}_{V_{t}}$ and $\sum_{t=1}^{T}\norm{x_{t}}_{V_{t}^{-1}}$, which are useful for proving \Cref{thm:sublinear_elimination,thm:acc-linTS}.
\begin{lemma}[Thm. 2 of \citep{abbasi2011improved}]\label{lemma:bound_theta}
  With \Cref{as:1} to \ref{as:3}, for the $\widehat \theta_{t}$ estimation according to \Cref{eq:theta_estimation} and for any $\delta > 0$, with probability at least $1 - \delta$ for all $t \ge 0$, we have
  \begin{align*}
    \norm{\widehat \theta_{t} - \theta^{*}}_{V_{t}} \le 1 + \sqrt{2\log\rbr{\frac{1}{\delta}} + d\log\rbr{1+\frac{T}{d}}}
  .\end{align*}
\end{lemma}

\begin{lemma}[Lemma 4 of \citep{abbasi2011improved}]\label{lemma:bound_x}
  Let $\cbr{x_{t}}$ be a sequence in $\RR^{d}$. For $V_{t} = I + \sum_{s=1}^{t-1}x_{s}x_{s}^{\top}$, we have
  \begin{align*}
    \sum_{t=1}^{T}\norm{x_{t}}_{V_{t}^{-1}}^{2} \le 2d \log\rbr{1 + \frac{T}{d}}
  .\end{align*}
\end{lemma}

\subsection{Proof of \Cref{lemma:bounded_regret}}

\begin{proof}
  Let $a_t^{*}$ be the optimal arm (whose feature is $x_{a_t^{*}}$) at time $t$ and suppose $\Psi_{s^{*}}$ is the set that contains $a^{*}$. Let $t_1$ be the time step that $a_{t}^*$ is placed to $\Psi_{s^{*}}$, and $t_2$ is the time that the played arm $a$ is placed in $\Psi_{s_t}$, we have
\begin{align*}
  x_{a_t^*}^{\top} \theta^{*} & \le x_{a_t^*}^{\top}\widehat\theta_{t_1} + \beta(\delta)\norm{x_{a_t^*}}_{V_{t_1}^{-1}} \\
                       & \le x_{a_t^*}^{\top}\widehat\theta_{t_1} + 2^{-s^{*}} \\
                       & \stackrel{(a)}{\le} \underline r + 2^{-s^{*}} + 2^{-s^{*}} \\
                       & \stackrel{(b)}{\le} x_{a}^{\top}\widehat\theta_{t_2} + 2^{-s_t} + 2\cdot2^{-s^{*}} \\
                       & \le x_{a}^{\top}\theta^{*} + \beta(\delta)\norm{x_{a}}_{V_{t_2}^{-1}}+ 2^{-s_t} + 2\cdot2^{-s^{*}} \\
                       & \le x_{a}^{\top}\theta^{*} + 2\cdot2^{-s_t} + 2\cdot2^{-s^{*}} \\
                       & \stackrel{(c)}{\le} x_{a}^{\top}\theta^{*} + 4\cdot2^{-s_t}
.\end{align*}
Inequality $(a)$ holds as $\underline r \ge x_{a_t^*}^{\top}\widehat\theta_{t_1} - 2^{-s^{*}}$ (since $\underline r$ is always greater than $x_{a_t^*}^{\top}\widehat\theta_{t_1} - 2^{-s^{*}}$ after $a_t^{*}$ advances to a $\Psi_{s^*}$); inequality $(b)$ holds as arm $a$ is not eliminated from $\Psi_{s_t}$; inequality $(c)$ holds as $s_t \le s^{*}$. This completes the proof.
\end{proof}

\subsection{Proof of \Cref{thm:sublinear_elimination}}
We first state the formal version of \Cref{thm:sublinear_elimination}.
\begin{theorem}[Formal version of \Cref{thm:sublinear_elimination}]\label{thm:formal_sublinear_elimination}
For any $\delta \in (0, 1)$, with probability at least $1- \delta$, the regret of \Cref{alg:sublinear_elimination} is bounded by
  \begin{align*}
    R(T) \le 16\beta(\delta/2)\sqrt{Td\log\rbr{1 + \frac{T}{d}}} + 64 \eta(T)\cdot T
  ,\end{align*}
  with $ \beta(\delta/2) = 1 + \sqrt{2\log\rbr{\frac{2}{\delta}} + d\log\rbr{1 + \frac{T}{d}}}$. $\eta(T) \in (0, 1)$ controls the approximate MIPS accuracy.
  
  The per-step time complexity is $K^{1 - \Theta(\frac{\eta(T)^4}{\log^2 T}) + o(\log^{-0.45}K)}$. The overall time complexity overhead (e.g., initialization) is $K^{1+o(1)}$.
\end{theorem}

\paragraph{{Proof of \Cref{thm:formal_sublinear_elimination} - time complexity}}

We break the time complexity into two parts:

\textbf{Overhead for maintaining $\Psi_{s}$:} This part contains the overhead induced by maintaining $\Psi_{s}$, which includes lines 5, 8-11, 19-24 of \Cref{alg:sublinear_elimination}. Line 5 is intializing all the initial $K$ arms, which takes $O(K \log K)$ for the heap related operations, and $O(\kappa\cdot K^{1 + o(1)})$ time to add all $a\in \Acal$ to the adaptive MIPS solver $\Mcal_{0}$. For lines 8-9 and 19-20, it only happens when an arm $a$ needs to be added (or advanced) to another $\Psi_{s}$. Notice that each arm
can only be added (or advanced) to a $\Psi_{s}$ for $\ceil{\log \frac{1}{8 \eta(T)}} + 1$ times. Therefore all the arms in total will induce an $\kappa\cdot K^{1+o(1)}\cdot\log\frac{1}{8\eta(T)}$ time complexity in overhead. Further for line 10-11 and 21-24, it only happens when an arm needs to be eliminated. Both the heap $\Hcal_{s}$ and the adaptive MIPS solver $\Mcal_{s}$ need to be updated, which in total induces an $\kappa\cdot K^{1+o(1)} + O(K\cdot \log K)$ overhead for all the arms (since all the arms can only be eliminated once). The overall overhead complexity is therefore $\kappa\cdot K^{1+o(1)}\cdot\log\frac{1}{8\eta(T)} + O( K \cdot\log K)$, rearranging the terms gives $K^{1+o(1)}$.

\textbf{Time complexity for selecting an arm:} This includes lines 18 and 26. With the construction of the adaptive MIPS solver $\Mcal$ (\Cref{alg:mips-adaptive}), the query time complexity is given by $\kappa \cdot K^{\rho_q + o(\log^{-0.45}K)}$. Plug in $(c, r, \epsilon) = (1/4, \frac{2^{-2s}(1-\eta(T)^{2})}{d^2\beta(\delta/2)^2}, \frac{2^{-2s}\eta(T)^{2}}{d^{2}\beta(\delta/2)^{2}})$ and $s \le \ceil{\log\frac{1}{8\eta(T)}}$, we have $\kappa =
K^{O(\frac{1}{\sqrt{\log K}})}$, and
$\rho_q = \frac{4c'^2}{(1+c'^2)^2}$, with $c' = 1 + \Theta(\frac{\eta(T)^2}{\log T})$ (see \Cref{thm:near-linear-MIPS}). Thus we have $\rho_q = 1 - \Theta(\frac{\eta(T)^4}{\log^2 T})$, which gives the per-step time complexity $K^{1 - \Theta(\frac{\eta(T)^4}{\log^2 T}) + o(\log^{-0.45}K)}$.

Therefore the overhead is $K^{1 + o(1)}$ and the per-step complexity is $K^{1 - \Theta(\frac{\eta(T)^4}{\log^2 T}) + o(\log^{-0.45}K)}$.

\paragraph{Proof of \Cref{thm:sublinear_elimination} - regret bound}

\begin{proof}
  With failure probability being $\frac{\delta}{2s_{max}}$ for \Cref{alg:mips-adaptive} and $\beta(\delta/2)$ in \Cref{alg:sublinear_elimination}, with probability at least $1-\delta$, all queries to \Cref{alg:mips-adaptive} (line 18 and line 26 of \Cref{alg:sublinear_elimination}) are answered correctly and \Cref{lemma:bound_theta} holds for all $t \in [T]$. Conditioning on those success events, we proceed to the regret bound.
  
  Suppose at time $t$, the played arm $a_t$ is chosen from set $\Psi_{s_{t}}$. We have that
\begin{align*}
  \beta(\delta/2)^{2}\inner{vec(x_{a_t}x_{a_t}^{\top})}{vec(V_{t}^{-1})} \le 2^{-2s_{t}}
.\end{align*}

By $\Cref{lemma:bounded_regret}$, we know that
\begin{align*}
    x_{a_t^*}^{\top}\theta^* - x_{a_t}^\top\theta^* \le 4 \cdot 2^{-s_t}
\end{align*}

When $s_t = \ceil{\log \frac{1}{8 \eta(T)}}$, we have $2^{-s_t} \le 8 \eta(T)$
  \begin{align}\label{eq:last_stage}
      x_{a_t^*}^{\top}\theta^{*} - x_{a_t}^{\top}\theta^{*} \le 32 \eta(T)
  \end{align}
  For stage $s_{t} < \ceil{\log \frac{1}{8 \eta(T)}}$, since the action $a_{t}$ is the result of querying $\Mcal_{s_{t}}$, we have
  \begin{align*}
    & \beta(\delta/2)^{2}\norm{x_{a_t}}_{V_{t}^{-1}}^{2} \ge \frac{1}{4}\cdot 2^{-2s_{t}-2} - \frac{5}{4}\eta(T)^{2} \implies \beta(\delta/2)\norm{x_{a_t}}_{V_{t}^{-1}} \ge 2^{-s_{t}-2}-2\eta(T)
  ,\end{align*}
  where we used the fact that $2^{-s_t-2} \ge 2\eta(T)$ and $\sqrt{a - b}\ge \sqrt{a} - \sqrt{b}$ for all $a \ge b$.
  Combining the results, we have
  \begin{align}\label{eq:previous_stage}
    x_{a_t^*}^{\top}\theta^{*} - x_{a_{t}}^{\top}\theta^{*} \le 4\cdot2^{-s_{t}} \le 16\beta(\delta/2)\norm{x_{a_{t}}}_{V_{t}^{-1}} + {32}\eta(T)
  .\end{align}
  Combining \Cref{eq:last_stage,eq:previous_stage} and summing over $t$, we have
  \begin{align*}
    R(T) & \le 16\beta(\delta/2)\sum_{t = 1}^{T}\norm{x_{a_{t}}}_{V_{t}^{-1}} + 32\eta(T)T + 32\eta(T)T \\
         & \le 16\beta(\delta/2) \sqrt{T d \log \rbr{1 + \frac{T}{d}}} + 64\eta(T) T
  .\end{align*}
  The second inequality is by \Cref{lemma:bound_x}. This completes the proof.

\end{proof}

\subsection{Proof of \Cref{thm:acc-linTS}}

We first state the formal version of \Cref{thm:acc-linTS}.

\begin{theorem}[Formal version of \Cref{thm:acc-linTS}]\label{thm:formal_acc-linTS}
  For any $\delta \in (0, 1)$, with probability at least $1-\delta$, the regret of \Cref{alg:accelerated-linTS} is bounded by
  \begin{align*}
  R(T) \le & \frac{4\gamma(\delta/4T)}{p}\rbr{\sqrt{2Td\log\rbr{1+\frac{T}{d}}} + \sqrt{8T \log\frac{4}{\delta}} }  \\
  & + \rbr{\gamma(\delta/4T) + \beta(\delta/4T)}\sqrt{2Td\log\rbr{1 + \frac{T}{d}}} \\
           & + \frac{6(1 + \gamma(\delta/4T) + \beta(\delta/4T))}{p}\cdot\frac{\eta(T)}{d}\cdot T
  ,\end{align*}
  where $\beta(\delta/4T) = 1 + \sqrt{2\log\frac{4T}{\delta} + d\log\rbr{1 + \frac{T}{d}}}$, $\gamma(\delta/4T) = \beta(\delta/4T)\sqrt{bd\log \frac{b'd}{\delta/4T}}$, with $b, b', p$ are constants defined in \Cref{def:D-TS}. $\eta(T) \in (0, 1)$ controls the approximate MIPS accuracy.

  The per-step time complexity is $K^{1 - \Theta(\eta(T)^2) + o(\log^{-0.45}K)}$. The time complexity of the data structure maintenance (line 4) is ${K^{1+o(1)}}$ which is paid once at initialization.
\end{theorem}

\paragraph{Proof of \Cref{thm:formal_acc-linTS} - time complexity}

We first show the per-step complexity. The per-step time complexity is $\kappa K^{\rho_q + o(\log^{-0.45}K)} \log\frac{d}{\eta(T)}$, as each query to \Cref{alg:mips-adaptive} has complexity $\kappa K^{\rho_q + o(\log^{-0.45}K)}$ and line 8 of \Cref{alg:accelerated-linTS} requires a binary search which induces another factor of $\log \frac{d}{\eta(T)}$. By setting $(c, r, \epsilon) = (1 - \frac{1}{i+1}, \frac{i\cdot\eta(T)}{d}, \frac{\eta(T)}{d})$, we have $\kappa = O(\log KT) = K^{O(\frac{1}{\sqrt{\log K}})}$ and $\rho = \frac{4c'^2}{(1+c'^2)^2}$, with $c' = 1 + \Theta(\eta(T))$ (see \Cref{thm:near-linear-MIPS}). It then implies $\rho_q = 1- \Theta(\eta(T)^2)$, which corresponds to the per-step time complexity in \Cref{thm:acc-linTS}. Notice that for Line 6 in \Cref{alg:accelerated-linTS}, adding new arms to and deleting arms from all $\Mcal_i$ takes at most $C_{change}K^{o(1)} \ceil{\frac{d}{\eta(t)}}$ time, which is negligible comparing with $K^{1 -\Theta(\eta(T)^2) + o(\log^{-0.45}K)}$.

Next we prove the preprocessing time complexity. By \Cref{thm:near-linear-MIPS}, the preprocessing time complexity is $\kappa K^{1 + o(1)}$. With $\kappa = K^{o(1)}\log T$, the complexity becomes $K^{1+o(1)}\cdot\log T$. Note that in line 5, $\ceil{\frac{d}{\eta(T)}}$ copies of \Cref{alg:mips-adaptive} are constructed. Therefore the preprocessing time complexity is $\frac{d K^{1 + o(1)}\log T}{\eta(T)}$ and the per-step complexity is $K^{1 - \Theta(\eta(T)^2) + o(\log^{-0.45}K)} \cdot \log T \cdot \log \frac{d}{\eta(T)}$, which can be further simplified as $K^{1 + o(1)}$ preprocessing complexity and $K^{1 - \Theta(\eta(T)^2) + o(\log^{-0.45}K)}$ per-step complexity.

\paragraph{Proof of \Cref{thm:formal_acc-linTS} - regret bound}

By setting the MIPS solver $\Mcal$'s success probability to be at least $1 - \frac{\delta \cdot \eta(T)}{2d}$, we have the all queries (line 9 of \Cref{alg:accelerated-linTS}) are answered correctly with probability at least $1 - \frac{\delta}{2}$.Further, note that with setting of $\gamma\rbr{\delta/4T}, \beta(\delta/4T)$, with probability at least $1-\frac{\delta}{2}$, for all $t \le T$, we have
\begin{align*}
  \norm{\widehat \theta_{t} - \theta^{*}}_{V_{t}} \le \beta(\delta/4T),\quad \norm{\widetilde \theta_{t} - \widehat \theta_{t}}_{V_{t}} \le \gamma(\delta/4T)
,\end{align*}
with the first inequality comes from \Cref{lemma:bound_x}, and the second one follows from the \textit{concentration} part of \Cref{def:D-TS}. The rest of the proof only considers the case when the events above hold, which happens with probability at least $1 - \delta$.

The regret analysis is similar to the one in \citep{abeille2017linear}. We start with the regret decomposition
\begin{align*}
  R(T) = \underbrace{\sum_{t=1}^{T}\rbr{x_{a_t^*}^{\top} \theta^{*} - x_{a_{t}}^{\top}\widetilde \theta_{t}}}_{R^{TS}(T)} + \underbrace{\sum_{t=1}^{T}\rbr{x_{a_{t}}^{\top}\widetilde \theta_{t} - x_{a_{t}}^{\top}\theta^{*}} 
}_{R^{RLS}(T)}
,\end{align*}
where the $R^{RLS}(T)$ is the regret induced by the ``regularized least square" estimation, and $R^{TS}(T)$ measures the regret of making decision based on the $\widetilde \theta_{t}$ drawn by TS.

\textbf{Bounding $R^{RLS}(T)$.}
\begin{align*}
  R^{RLS}(T) & \le \sum_{t=1}^{T}\abs{x_{a_{t}}^{\top}\rbr{\widetilde \theta_{t} - \widehat \theta_{t}}} + \sum_{t=1}^{T}\abs{x_{a_{t}}^{\top}\rbr{\widehat \theta_{t} - \theta^{*}}} \\
             & \le \sum_{t=1}^{T}\norm{x_{a_{t}}}_{V_{t}^{-1}}\rbr{\norm{\widetilde \theta_{t} - \widehat \theta_{t}}_{V_{t}} + \norm{\widehat \theta_{t} - \theta^{*}}_{V_{t}}} \\
             & \le \rbr{\gamma(\delta/4T) + \beta(\delta/4T)} \sum_{t=1}^{T}\norm{x_{a_{t}}}_{V_{t}^{-1}} \\
             & \le \rbr{\gamma(\delta/4T) + \beta(\delta/4T)} \sqrt{2Td\log \rbr{1 + \frac{T}{d}}}
.\end{align*}
The last inequality follows from \Cref{lemma:bound_x}.

\textbf{Bounding $R^{TS}(T)$.}
At time $t$, denote $J_t(\theta) \coloneqq \max_{a\in\Acal}x_{a}^{\top}\theta$. Suppose \Cref{alg:accelerated-linTS} selects $a_{t}$. Define $\Delta_{t} \coloneqq J_t(\widetilde \theta_{t}) - x_{a_{t}}^{\top}\widetilde \theta_{t}$, which is the approximation error solving MIPS approximately for $\widetilde \theta_{t}$. Denote $R_{t}^{TS} = x_{a_t^*}^{\top}\theta^{*} - x_{a_{t}}^{\top}\widetilde \theta_{t}$, we have
\begin{align*}
  R_{t}^{TS} \le J_t(\theta^{*}) - J_t(\widetilde \theta_{t}) + \Delta_{t}
.\end{align*}
Define $\Ccal_{t} = \cbr{\theta~|~\norm{\theta - \widehat \theta_{t}}_{V_{t}} \le \gamma(\delta/4T)}$, which implies $\widetilde \theta_{t} \in \Ccal_{t}$ for all $t \in [T]$. Then
\begin{align*}
  R_{t}^{TS} \le J_t(\theta^{*}) - \inf_{\theta\in\Ccal_{t}}J_t(\theta) + \Delta_{t}
.\end{align*}
We denote $\widetilde \theta_{t}$ is \textit{optimistic} if $J_t(\widetilde \theta_{t}) \ge J_t(\theta^{*})$. For any step $t$, $\widetilde \theta_{t}$ is optimistic with probability at least $p/2$, where $p$ is defined in \Cref{def:D-TS} (see Lemma 3 of \citep{abeille2017linear}). Condition on $\widetilde \theta_{t}$ being optimistic, we have
\begin{align*}
  R_{t}^{TS} & \le J_t(\widetilde \theta_{t}) - \inf_{\theta\in\Ccal_{t}}J_t(\theta) + \Delta_{t} \\
             & \le x_{a_{t}}^{\top}\widetilde \theta_{t} - \inf_{\theta\in\Ccal_{t}}\max_{a\in\Acal}x_{a}^{\top}\theta + 2\Delta_{t} \\
             & \le x_{a_{t}}^{\top}\widetilde \theta_{t} - \inf_{\theta\in\Ccal_{t}}x_{a_{t}}^{\top}\theta + 2\Delta_{t} \\
             & \le \sup_{\theta\in\Ccal_{t}}\norm{x_{a_{t}}}_{V_{t}^{-1}}\norm{\widetilde \theta_{t} - \theta}_{V_{t}} + 2\Delta_{t} \\
             & \le 2\gamma(\delta/4T) \norm{x_{a_{t}}}_{V_{t}^{-1}} + 2\Delta_{t}
.\end{align*}
Note that the right-hand side is always positive, taking expectation with regard to $\widetilde \theta_{t}$ we have
\begin{align*}
  R_{t}^{TS} &\le \EE_{\widetilde \theta_{t}}\sbr{2\gamma(\delta/4T) \norm{x_{a_{t}}}_{V_{t}^{-1}} + 2\Delta_{t}\Bigm\vert \widetilde \theta_{t} \text{ is optimistic }} \\
             &\le \frac{2}{p}\EE_{\widetilde \theta_{t}}\sbr{2\gamma(\delta/4T) \norm{x_{a_{t}}}_{V_{t}^{-1}} + 2\Delta_{t}} 
.\end{align*}
Next we proceed to bound $\Delta_t$. Note that as $\norm{\theta^*}_2 \le 1$, $\norm{\widehat \theta_{t} - \theta^{*}}_{V_{t}} \le \beta(\delta/4T),\norm{\widetilde \theta_{t} - \widehat \theta_{t}}_{V_{t}} \le \gamma(\delta/4T)$. For all $t\in[T]$, the multiplicative error (introduced by MIPS according to the parameter in \Cref{alg:accelerated-linTS}, line 9) for $J_t(\widetilde \theta_t)$ is at most $(1 + \beta(\delta/4T) + \gamma(\delta/4T))\frac{\eta(T)}{d}$ and the additive error $\epsilon$ induces another $2(1+\beta(\delta/4T) + \gamma(\delta/4T))\frac{\eta(T)}{d}$ approximation error.

Therefore for all $t \in [T]$, we have
\begin{align*}
  \Delta_{t} \le 3\rbr{1 + \gamma(\delta/4T) + \beta(\delta/4T)}\frac{\eta(T)}{d}
.\end{align*}
It thus implies
\begin{align*}
  R^{TS}(T) \le & \frac{4\gamma(\delta/4T)}{p}\sum_{t=1}^{T}\EE_{\widetilde \theta_{t}}\sbr{\norm{x_{a_{t}}}_{V_{t}^{-1}}} \\
                & + \frac{6\rbr{1 + \gamma(\delta/4T) + \beta(\delta/4T)}}{p}\eta(T) T\\
  \le & \frac{4\gamma(\delta/4T)}{p}\rbr{\sqrt{2Td\log\rbr{1+\frac{T}{d}}} + \sqrt{8T \log\frac{4}{\delta}} } \\
                & + \frac{6\rbr{1 + \gamma(\delta/4T) + \beta(\delta/4T)}}{p}\cdot\frac{\eta(T)}{d}\cdot T
.\end{align*}
The second inequality follows from Azuma's inequality on bounding the difference between $\sum_{t=1}^{T}\EE_{\widetilde \theta_{t}}\sbr{\norm{x_{a_{t}}}_{V_{t}^{-1}}}$ and $\sum_{t=1}^{T}\norm{x_{a_{t}}}_{V_{t}^{-1}}$. Combining the bound for $R^{RLS}(T)$ and $R^{TS}(T)$ completes the proof.

%% file: apdx_experiments.tex
\section{Deferred Experiment Results}\label{apdx:exp}

Here we present the deferred experiment results.

\paragraph{Synthetic Experiment - Addition and Deletion} We empirically evaluate the performance of Sub-TS when there are both arm additions to and deletions from $\Acal$. The environment is set as specified in \Cref{sec:experiment}. Further, we set the number of arms $K$ to be 10,000. For every 20 time steps, there are $2$ arms newly generated from the unit spherical Gaussian distribution, and $2$ random arms in $\Acal$ get deleted. The time horizon is set to 20,000 and the results are in \Cref{table:addition_deletion}.

\begin{table}[!h]
\centering
\begin{tabular}{@{}c|cccc@{}}
\toprule
Algorithm & Linear TS & Sub-TS, shortlist 10  & Sub-TS, shortlist 30  & Sub-TS, shortlist 100  \\ \midrule
Regret    & $612 \pm 43$    & $640 \pm 42$   &  $612 \pm 43$  &  $612 \pm 43$      \\
Time(s)   & $44.80$ &  $25.49~(25.23)$  &  $26.37~(26.11)$ &  $29.69~(29.43)$    \\
Speedup  & $\times 1$   &  $\bf{\times 1.75~(\times 1.77)}$  &  $\bm{\times 1.70~(\times 1.72)}$ &  $\bm{\times 1.51~(\times 1.52)}$ \\ \bottomrule
\end{tabular}
\caption{\textbf{Synthetic Experiment - Addition and Deletion.} The algorithms and ``Regret", ``Time" and ``Speedup" are defined the same as in \Cref{table:different_K}. We see that the Sub-TS is able to handle arms' changing, including both additions and deletions, and delivers around 1.51 -- 1.77 times speedup.}
\label{table:addition_deletion}
\end{table}

\paragraph{Synthetic Experiment - Impact of $C_{change}$} Here we empirically evaluate the impact of different numbers of arms' changing. The environment is set as specified in \Cref{sec:experiment}. Further, we set the initial number of arms to be 10,000. For every 20 time steps, there are $C_{change}$ arms newly generated from the Gaussian distribution and included into $\Acal$. The time horizon is set to 20,000 and the results are in \Cref{table:different_C_change}.

\begin{table}[!h]
\centering
\begin{tabular}{@{}cc|cc|cc@{}}
\toprule
                                                        &          & Linear Elim & Sub-Elim, shortlist 30 & Linear TS & Sub-TS, shortlist 30 \\ \midrule
\multicolumn{1}{c|}{\multirow{3}{*}{$C_{change}=2$}}   & Regret   &          $4433 \pm 399$   &          $4393 \pm 392$             &     $566 \pm 52$     &     $566 \pm 52$    \\
\multicolumn{1}{c|}{}                                   & Time (s) &      $35.72$       &   $2.85~(2.10)$      &      $45.54$    &   $21.01~(20.76)$  \\
\multicolumn{1}{c|}{}                                   & Speedup &       $\times 1$      &    $\bm{\times 12.53~(\times17.01)}$   &  $\times 1$  &   $\bm{\times 2.17~(\times 2.19)}$  \\ \midrule

\multicolumn{1}{c|}{\multirow{3}{*}{$C_{change}=10$}}  & Regret   &  $4428 \pm 224$  &  $4345 \pm 244$   &  $639 \pm 54$  &  $638 \pm 54$    \\
\multicolumn{1}{c|}{}                                   & Time (s) &  $36.22$ &   $3.04~(2.30)$   &   $55.99$  &   $24.91~(24.65)$     \\
\multicolumn{1}{c|}{}                                   & Speed-up &  $\times 1$   &  $\bm{\times 11.91~(\times 15.75)}$    & $\times 1$  &  $\bm{\times 2.25~(2.27)}$   \\ \midrule
\multicolumn{1}{c|}{\multirow{3}{*}{$C_{change}=50$}} & Regret   &  $4106 \pm 154$   &      $4062 \pm 169$     &  $581 \pm 45$   &  $619 \pm 62$   \\
\multicolumn{1}{c|}{}                                   & Time (s) &  $36.59$   &    $3.94~(3.20)$    &   $106.96$    &   $48.87~(48.62)$      \\
\multicolumn{1}{c|}{}                                   & Speedup &     $\times 1$        &    $\bm{\times 9.28~(\times 11.43)}$    &   $\times 1$    &  $\bm{\times 2.19~(\times 2.20)}$     \\ \bottomrule
\end{tabular}
\vspace{+0.5em}
\caption{\textbf{Synthetic Experiment - Impact of Different $C_{change}$.} The algorithms and ``Regret", ``Time" and ``Speedup" are defined the same as in \Cref{table:different_K}. Notice that Linear Elim and Sub-Elim are not much affected by $C_{change}$, as they will have already removed many arms in the later stages, and therefore the newly added arms do not affect the running time by much. The running time of Linear TS and Sub-TS, however, is significantly affected by $C_{change}$, as they are running on an increasingly large arm set. The  Despite the impact on their individually running time, our algorithms are shown to deliver stable speedup in all evaluated settings.}
\label{table:different_C_change}
\vspace{-1em}
\end{table}